\documentclass[twoside]{article}

\usepackage[preprint]{aistats2027}
\usepackage[round]{natbib}

\usepackage{xcolor}
\usepackage[hidelinks]{hyperref}
\hypersetup{
    colorlinks,
    linkcolor={blue!50!black},
    citecolor={blue!50!black},
    urlcolor={blue!50!black}
}

\usepackage{url}
\usepackage{subcaption, booktabs, multicol, multirow}
\usepackage[colorinlistoftodos,prependcaption,textsize=tiny]{todonotes}

\usepackage{amsmath, amssymb, amsthm, bbm, mathtools}
\usepackage{algorithm, algpseudocode}

\def\P{{\mathbb{P} }}
\def\E{{\mathbb{E} }}
\def\MLE{{\mathrm{MLE}}}

\def\T{\mathrm{\scriptscriptstyle{T}}}
\DeclareMathOperator*{\argmax}{arg\,max}

\newcommand{\norm}[1]{\Vert#1\Vert}
\theoremstyle{plain}
\newtheorem{theorem}{Theorem}

\newtheorem{remark}{Remark}
\newtheorem{corollary}{Corollary}

\newtheorem{proposition}{Proposition}

\begin{document}

%
\runningtitle{List ordered graphical models}

%

\twocolumn[

\aistatstitle{Covariate-dependent Joint Modeling of Multivariate Ordinal Preferences and Its Connections with Comparison Models}

\aistatsauthor{ Yujie Chen \And Antik Chakraborty \And  Anindya Bhadra }

\aistatsaddress{Department of Statistics \\Purdue University\\ chen1866@purdue.edu \And Department of Statistics \\Purdue University\\ antik015@purdue.edu \And Department of Statistics \\Purdue University\\ bhadra@purdue.edu} ]

\begin{abstract}

Multivariate ordinal data along with covariates are commonly collected in problems ranging from alignment of language models with human preferences, as well as in recommender systems. For example, data sets such as MovieLens  contain several movies rated on a scale 1--5 by human users, along with their demographic information such as age or gender. Similarly, data sets such as HelpSteer  collect human feedback on several attributes such as \emph{helpfulness} or \emph{verbosity} of LLM response on an ordinal scale, with covariates depending on the LLM prompt--response pairs. Unfortunately, the standard approaches for modeling these data (a) look at the attributes individually rather than jointly, and (b) often convert the data into pairwise or list-wise win--loss comparisons for fitting models such as Bradley--Terry and Plackett--Luce. Both of these lead to a coarsening of what is actually observed, which we address via a joint covariate-dependent consecutive ratio Markov random field model. We also show pairwise or listwise comparison models are obtained under  restrictions of our joint model, and that joint modeling improves comparisons. We also develop a maximum likelihood inference procedure even in the presence of an intractable normalizer. 
\end{abstract}

\section{\uppercase{Introduction}}

Multivariate ordinal data are routinely collected in many fields, 
e.g., in alignment of language models with human preferences, or in recommender systems. Examples include  human annotation data in large language model (LLM) outputs on attributes such as  \emph{helpfulness} or \emph{verbosity} as in HelpSteer  \citep{wang2024helpsteer} and movie ratings data in MovieLens \citep{harper2015movielens}, along with demographic covariates. The purpose of this article is to simultaneously address two common shortcomings of the approaches for modeling such data: (a) a joint model instead of separate models of each ordinal response and (b) a model for the ordinal responses themselves, instead of a reduced form derived summary such as pairwise or list-wise win--loss metrics. Throughout, our proposed approach incorporates available covariate information.

The classical models used to analyze multivariate ordinal data are essentially extensions of univariate ordinal regression models. Broadly, the univariate models are of two categories: odds ratio based models and latent variable models \citep{agresti2010analysis}. Their multivariate extensions also follow similar strategies, in that primarily pairwise and marginal odds ratios are parameterized \citep{bartolucci2007extended}. Alternative methods consider extensions of the latent variable approach to multivariate settings \citep{hirk2019multivariate}, but face difficulties in scaling to high dimensions, and may not lend themselves to straightforward interpretations. 

Other approaches focus not on the ordinal data per se, but their reduced form summaries as given by pairwise or list-wise win--loss rankings.
Popular models include the Bradley--Terry (BT) model for pairwise  \citep{bradley1952rank} and the Plackett--Luce (PL) model for list-wise comparisons \citep{plackett1975analysis, luce1959individual}, sometimes with provisions for handling ties \citep{davidson1970extending}. These methods apply directly to comparison outcomes, often ignoring the underlying ordinal scores that generated these preferences \citep{aoki2024universe}. This may result in potential loss of information, e.g. consider the case where item $A$ receives a score of $0$, $B$ receives $5$ while $C$ receives a score of $4$ in an ordinal scale of $0$--$5$. Then, pairwise comparisons between $(A, B)$ and $(B,C)$ yield the same result that $B$ wins against both $A$ and $C$, ignoring the fact that B--C are much closer than A--B. List-wise comparisons lead to similar loss of information. Further, in data sets such as HelpSteer, it is reasonable to assume ordinal attributes are inherently globally coupled. For example, a human user may rate an otherwise accurate LLM response \emph{unhelpful} if it is perceived overly \emph{verbose}. Separate modeling of the ordinal attributes ignores this coupling.

In the current work we address both of these shortcomings by directly modeling the multivariate ordinal scores in a covariate-dependent manner using a consecutive ratio Markov random field (CR--MRF) model. We show that the univariate ordinal models, as well as pairwise or list-wise comparison models are obtained under restrictions of our model, before proceeding to develop inferential strategies.

\subsection{Related Works}

Undirected graphical models \citep{koller2009probabilistic} or Markov random fields (MRF) provide a natural way to describe dependencies between multiple random variables through conditional independence relationships. They have been successfully used in inference of cellular networks \citep{friedman2004inferring}, spatial statistics \citep{besag1974spatial}, natural language processing \citep{lafferty2001conditional}, among others. Building on this line of work, \cite{suggala2017ordinal} consider extensions of cumulative/continuation/consecutive odds ratio models \citep{mccullagh1989generalized} for univariate ordinal data to the multivariate setting. Since undirected graphical models specify a joint distribution over the set of random variables through simpler node-conditional distributions, \cite{suggala2017ordinal} show that a node-conditional consecutive odds ratio model leads to a valid joint distribution over the variables. In this article, we adopt this MRF approach to address cases where the conditional dependencies between variables may potentially change due to covariate effects. Moreover, while \cite{suggala2017ordinal} employ a pseudo-likelihood approach for inference on model parameters thus avoiding the intractable normalizing constant, we develop a full-likelihood based analysis using the importance sampling estimator of the intractable normalizing constant developed in \cite{chen2024likelihood, chen2026exact}. Further, and most importantly, we show that joint modeling is useful even when the primary interest lies in pairwise or list-wise comparisons, by comparing against BT and PL models, which \cite{suggala2017ordinal} do not address.

Several works \citep{obrien2004bayesian, Bhat2010Comparison, kenne2016pairwise, hirk2019multivariate, vana2024multivariate} have addressed the problem of multivariate ordinal regression from a latent variable approach. Specifically, their models assume a latent multivariate Gaussian/logistic random variable such that its discretization on appropriate thresholds leads to the observed ordinal outcomes. For inference, they primarily employ the composite likelihood approach. As with any latent variable model of this type, although specification of the dependence structure is straightforward in the latent space, how it translates to dependence between the observed outcomes is relatively opaque. 
\subsection{Our Contributions}
\begin{enumerate}
    \item Building on the consecutive-ratio ordinal graphical model (CR--MRF) studied by \cite{suggala2017ordinal}, we develop in Section~\ref{sec:crmrf} a \emph{covariate-dependent} ordinal graphical model where node and edge parameters are functions of covariates. We also show how to perform statistically efficient \emph{maximum  likelihood} inference in this model, while \cite{suggala2017ordinal} consider \emph{pseudolikelihood} inference to bypass the intractable normalizer.
    \item We explicitly demonstrate in Section~\ref{sec:conn} that comparison models such as BT and PL can be recovered from the CR--MRF under specific restrictions, which clarifies the coarsening of the observations performed by BT or PL. To our knowledge, this connection has not been noted either in the MRF or in the comparison models literature.
    \item The simple node-conditional distributions of the covariate-dependent CR--MRF allow for straightforward extension of the methodology for handling missing data; and we provide a brief overview in Supplementary Section \ref{sec:missing_data}. In contrast, the comparison models require a strongly connected graph over the items to proceed with maximum likelihood estimation.
    \item Through numerical and data experiments using \emph{MovieLens}, we demonstrate that a joint model across the items improves predictions of pairwise or list-wise comparisons. 
\end{enumerate}
\section{\uppercase{Background on Ordinal Comparison Models}}\label{sec:background}
We begin by providing a brief overview of the Bradley--Terry (BT) and Plackett--Luce (PL) models. Consider a list of $p$ items indexed by $ [p] \coloneqq  \{1, \ldots, p\}$. When a pair $(j,k) \subset [p]$ is compared with scores $Y_j, Y_k$, the classical BT model \citep{bradley1952rank} posits:
\begin{equation}\label{eq:BT_model}
    \P[Y_j < Y_k] = \dfrac{\gamma_j}{\gamma_j + \gamma_k},
\end{equation}
where $\gamma_j>0$ is interpreted as the strength associated with item $j$ and it is assumed that lower scores are better.
In another parameterization, $u_j \coloneqq  \log \gamma_j$ can also be thought of as the utility of item $j$. In the original formulation of the BT model, the scores $Y_j, Y_k$ are generally unobserved but the comparison outcome is observed.

The PL model \citep{plackett1975analysis, luce1959individual} generalizes the BT model to multiple comparisons. Suppose a subset of items $ S = \{i_1, \ldots, i_k\} \subset [p]$ where $2 \leq k \leq p$ is compared with scores $\{Y_{i_1}, \ldots, Y_{i_k}\}$ and the outcome is a preference of the items $Y_S \in \{i_{\pi(1)}, \ldots, i_{\pi(k)}: \pi \in S_k\}$ with $S_k$ being the symmetric group on $[k]$. Here $i_{\pi(j)}$ is the item in $S$ ranked $j$. The Plackett-Luce model posits:
\begin{equation}\label{eq:PL_model}
    \P(Y_{i_{\pi(1)}} < Y_{i_{\pi(2)}} < \cdots <Y_{i_{\pi(k)}}) = \prod_{r = 1}^{k-1} \dfrac{\gamma_{i_\pi(r)}}{\sum_{l = r}^k \gamma_{i_\pi(l)}}.
\end{equation}
There have been extensions to the basic BT and PL models to accommodate ties and weak rankings, respectively. For BT, \cite{davidson1970extending} proposed the following Bradley--Terry--Davidson (BTD) modification:
\begin{align}\label{eq:BTD_model}
    \P[Y_j < Y_k] &= \dfrac{\gamma_j}{\gamma_j + \gamma_k + \nu \sqrt{\gamma_j \gamma_k}} \nonumber, \\
    \P[Y_j = Y_k] &= \dfrac{\nu\sqrt{\gamma_j\gamma_k}}{\gamma_j + \gamma_k + \nu \sqrt{\gamma_j \gamma_k}}
\end{align}
where $\nu\geq 0$ controls the tie propensity. Similarly, for PL, letting $A$ denote the subset of items yet to be ranked, and considering items $S \subseteq A$ to be ranked identically, leads to probability:
\begin{equation}\label{eq:DL_model}
    \P(S \mid A) =  \dfrac{\delta_{|S|}\left(\prod_{j \in S}\gamma_j\right)^{1/|S|}}{\sum_{\varnothing \neq B \subseteq A}\delta_{|B|}\left(\prod_{j \in B}\gamma_j\right)^{1/|B|} }.
\end{equation}
Here for any $A$, $|A|$ denotes its cardinality and $\delta_{|A|}$ is the propensity of $|A|$ items to be ranked the same. In other words, $\P(S \mid A) \propto w_{PL}(S)$ where $w_{PL}(S)$ is the stage weight of the set $S$ and $w_{PL}$ has the expression in the numerator of \eqref{eq:DL_model}. 
\section{\uppercase{The Intractable CR--MRF model and ML Inference}}\label{sec:crmrf}
\subsection{The Homogeneous CR--MRF}
\label{sec:ord-homog}
Let $Y = (Y_{1}, \ldots, Y_{p})^{\top} \in \mathcal{Y} = \{0, 1, \ldots, M\}^{p} $ denote a vector of $p$ ordinal variables, each taking values on a $(M{+}1)$-level ordered scale. The consecutive-ratio Markov random field (CR--MRF) introduced by \cite{suggala2017ordinal} has the probability density function:
\small
\begin{align}
\label{eq:crmrf-joint}
  p(y \mid \Theta) 
  &\propto \exp\left(\sum_{j=1}^p \sum_{m = 1}^M \theta_{j,m}^{\mathrm{node}} \phi_{j,m}(y_j) 
   +  \sum_{j<k}^p \theta_{jk}^{\mathrm{edge}}  \phi_{jk} (y_j, y_k)\right),
\end{align}
\normalsize
where the node-sufficient statistic is $\phi_{j,m}(y_j) = \mathbf{1}(y_j\leq m-1)$ and edge-sufficient statistic is $\phi_{jk}(y_j, y_k) = (M- y_j)(M - y_k)/M^2$. The model is parametrized by $\Theta = ({\Theta^{\mathrm{node}}}, {\Theta^{\mathrm{edge}}})$ where $\Theta^{\mathrm{node}} = (\theta^{\mathrm{node}}_{j,m}) \in \mathbb{R}^{p \times M}$ and $\Theta^{\mathrm{edge}} = (\theta^{\mathrm{edge}}_{jk}) \in \mathbb{R}^{p \times p}$. The edge parameter matrix $\Theta^{\mathrm{edge}} $ is symmetric
with zero diagonal. Thus, node strengths are allowed to change between levels of the ordinal variable, but the dependency between the variables is fixed over all levels of the variable. 
The analytically intractable normalizing constant of the model is  $Z(\Theta) = \sum_{y} q_{\Theta}(y)$, where $q_{\Theta}(y)$ is the unnormalized density in \eqref{eq:crmrf-joint}. Since the sample space of the node and edge-specific sufficient statistics are finite, $Z(\Theta)$ is finite for any real $\Theta$. 

The main feature of \eqref{eq:crmrf-joint} is that it extends the univariate consecutive ratio model where the main model assumption is that $\log [\P(Y_j = m-1)/\P(Y_j = m)] = \theta_{j,m}$ for $m = 1, \ldots, M$. For \eqref{eq:crmrf-joint}, a similar expression is available for each conditional consecutive ratio, i.e. $\log [\P(Y_j = m-1\mid Y_{-j})/\P(Y_j = m \mid Y_{-j})]$. We refer to Supplementary Section \ref{sec:crmrf-extra} for detailed expressions. 
\begin{remark} [Bounded Sufficient Statistic] Our parameterization differs from  \citet{suggala2017ordinal} in normalizing the edge statistic $(M - y_{j})(M - y_{k})/M^{2}$ to lie in $[0, 1]$. For every pair, the edge statistic attains its maximum at $(y_{j}, y_{k}) = (0, 0)$ and its minimum at any pair containing the highest level $M$. We do this for numerical stability, noting that there is a one-to-one correspondence between the two formulations.
\end{remark}
     
\subsection{Covariate-dependent CR--MRF}
Suppose we observe covariates $X$ along with the multivariate ordinal variable $Y$. In order to study the effect of the covariates on $Y$, a regression model specifies a conditional distribution for $Y \mid X$. Given the homogeneous CR--MRF model introduced earlier, a natural way to incorporate this dependence is by assuming $Y\mid X \sim \text{CR--MRF}(\Theta(X))$. We model this as follows.

Suppose for each subject $i = 1, \ldots, n$, we observe covariate information $x_{i} = (x_{i,1}, \ldots, x_{i,q})^{\top} \in \mathbb{R}^{q}$. We model each node and edge parameter as a linear function of $x_{i}$,
\begin{align}
\label{eq:cov-coef-node}
  \theta_{j,m}(x_{i})
  \;=\;
  \theta_{j,m,0}
  \;+\;
  \sum_{v=1}^{q} \theta_{j,m, v}\, x_{i,v}, \nonumber \\
  \theta_{jk}(x_{i})
  \;=\;
  \theta_{jk,0}
  \;+\;
  \sum_{v=1}^{q} \theta_{jk,v}\, x_{i,v},  
\end{align}
for $j = 1,\ldots, p$ and $m = 1, \ldots M$ with the convention $\theta_{jk,v} = \theta_{kj,v}$ and $\theta_{jj,v} = 0$ for all $v$. Each $\theta_{j,m,0}$ and $\theta_{jk,0}$ is an intercept that captures the node or edge parameter at the reference covariate level $x_{i} = 0$, and each $\theta_{j,m,v}$ and $\theta_{jk,v}$ for $v = 1, \ldots, q$ is the slope of covariate $v$. 
 
The joint distribution of ordinal scores $Y_i$ for individual $i$ is $p(y_{i} \mid x_{i}, \Theta) = q_{\Theta(x_i)}(y_i)/Z(\Theta(x_i))$.  
The joint distribution across $n$ independent individuals factorizes as $p(y \mid x, \Theta) = \prod_{i=1}^{n} p(y_{i} \mid x_{i}, \Theta)$. The homogeneous CR--MRF \eqref{eq:crmrf-joint} is recovered as the special case in which all slopes vanish, i.e., $\theta_{j,m,v} = 0$ and $\theta_{jk,v} = 0$ for $v \ge 1$.

\subsection{The Tractable Independence Model}
\label{sec:ord-indep}
A key challenge with \eqref{eq:crmrf-joint}  is that $Z(\Theta)$ is intractable. The problem is exacerbated in a covariate-dependent setting, as one must now estimate $Z(\Theta(x_i))$ separately for $i=1,\ldots,n$. However, as noted by \citet{chen2024likelihood}, 
setting all edge parameters to zero in \eqref{eq:crmrf-joint} eliminates the cross-terms and yields a product distribution that is fully tractable. In the homogeneous case,
\begin{align}
\label{eq:indep-homog}
  p_{0}(y \mid \Theta^{\mathrm{node}})
  &\;=\;
  \prod_{j=1}^{p} p_{0}(y_{j} \mid \Theta^{\mathrm{node}}_{j\bullet}), \nonumber \\
  p_{0}(y_{j} = m \mid \Theta_{j\bullet}^{\mathrm{node}})&
  \;=\;
  \frac{\exp\!\big\{\sum_{r=m+1}^{M} \theta^{\mathrm{node}}_{j,r}\big\}}
       {\sum_{\ell=0}^{M} \exp\!\big\{\sum_{r=\ell+1}^{M}\theta^{\mathrm{node}}_{j,r}\big\}},
\end{align}
where $\Theta_{j\bullet}^{\mathrm{node}}$ is the $j$-th row of $\Theta^{\mathrm{node}}$. The distribution has closed-form normalizer:
\begin{equation}
\label{eq:Z0-homog}
  Z_{0}(\Theta^{\mathrm{node}})
  \;=\;
  \prod_{j=1}^{p}
    \sum_{\ell=0}^{M}
      \exp\!\left\{\sum_{r=\ell+1}^{M} \theta^{\mathrm{node}}_{j,r}\right\}.
\end{equation}
Sampling from $p_{0}$ is easy since it is a distribution of independent categorical variables and can be performed in batches. This is crucial to our inference procedure. 
 
In the covariate-dependent case, setting $\theta_{jk,v} = 0$ for all $(j, k, v)$ yields a similar factorization within each individual.
For $i = 1, \ldots, n$,
\begin{equation}
\label{eq:indep-cov}
  p_{0}(y_{i} \mid x_{i}, \Theta^{\mathrm{node}})
  \;=\;
  \prod_{j=1}^{p} p_{0}\!\left(y_{i,j} \mid \Theta_{j\bullet}^{\mathrm{node}}(x_{i})\right),
\end{equation}
where $\Theta_{j \bullet}^{\mathrm{node}}(x_{i}) = (\theta^{\mathrm{node}}_{j,1}(x_{i}), \ldots, \theta^{\mathrm{node}}_{j,M}(x_{i}))$  is the individual-specific node parameter vector with components defined in \eqref{eq:cov-coef-node} for the $j$-th ordinal variable. The batchable computation of the individual-specific normalizer:
\begin{equation}
\label{eq:Z0-cov}
  Z_{0}(\Theta^{\mathrm{node}}(x_{i}))
  \;=\;
  \prod_{j=1}^{p}
    \sum_{\ell=0}^{M}
      \exp\!\left\{
        \sum_{r=\ell+1}^{M} \theta_{j,r}^{\mathrm{node}}(x_{i})
      \right\},
\end{equation}
is again closed-form, and the joint normalizer across $n$ individuals factorizes as $\prod_{i=1}^{n} Z_{0}(\Theta^{\mathrm{node}}(x_{i}))$. 

\subsection{Maximum Likelihood Inference}\label{sec:mle}
We now proceed to describe  maximum likelihood inference for the proposed model, a departure from pseudolikelihood deployed by previous approaches \citep{suggala2017ordinal}. Maximum likelihood is known to be statistically efficient in large samples \citep{rao1945information}, and yet, the presence of an intractable normalizer but simple node-conditional distributions naturally pointed to pseudolikelihood prior to the current work. We provide conditions under which the MLE exists for the CR--MRF model in Section \ref{sec:crmrf-extra} of the supplement.

Hereafter, we shall assume that the MLE exists. Since later we show the distribution to be minimal (Section \ref{subsec:cr-coarsening-proof}), it is also unique due to the fact that the Hessian of the negative log-likelihood yields the covariance matrix of the sufficient statistics which is positive definite implying strict log-concavity of the log-likelihood. We now proceed to detail a (stochastic) gradient ascent algorithm to obtain the MLE. Here again, for simplicity, we work with the homogeneous case and then note the key differences in the covariate-dependent case. Consider observations $y_1, \ldots, y_n \overset{iid}{\sim} p(\cdot \mid \Theta)$ where $p(\cdot \mid \Theta)$ is given in \eqref{eq:crmrf-joint}, then the MLE is defined as: 
\begin{align}\label{eq:MLE}
   \hat{\Theta}_{\MLE}  = \argmax_{\Theta}  \left[\sum_{i=1}^n \log q_\Theta(y_i) - n \log Z(\Theta)\right].
\end{align}
Direct optimization of \eqref{eq:MLE} is not possible due to the intractable $Z(\Theta)$. We instead adopt the stochastic optimization scheme based on an importance sampling estimate of $\nabla Z(\Theta)$ and $Z(\Theta)$ of \cite{chen2024likelihood}. In particular, we use the following iterative steps:
\begin{equation}\label{eq:mle_updates}
    \Theta_{t+1} = \Theta_t + \gamma_t\left[\sum_{i=1}^n \nabla \log q_{\Theta_t}(y_i) - n T(\Theta_t)\right], 
\end{equation}
where $T(\Theta)$ is a ratio-estimator of $\nabla \log Z(\Theta) = \nabla Z(\Theta)/Z(\Theta)$ and $\gamma_t>0$ is a schedule of decreasing step-sizes. Formally, we define $T(\Theta) = T_1(\Theta)/T_2(\Theta)$ where $T_1(\Theta)$ and $T_2(\Theta)$ are unbiased and consistent Monte Carlo estimators of $\nabla Z(\Theta)$ and $Z(\Theta)$, respectively. These estimators are constructed using the identities \citep{geyer1992constrained} $\E_{V\sim p_\Phi}[q_\Theta(V)/q_\Phi(V)] = Z(\Theta)/Z(\Phi)$ and $\E_{V\sim p_\Phi}[\nabla q_\Theta(V)/q_\Phi(V)] = \nabla Z(\Theta)/Z(\Phi)$ for any $\Phi$ for which $Z(\Phi)<\infty$ (in this case $\Phi$ can be any real-valued matrix). The main feature of the Monte Carlo scheme of \cite{chen2024likelihood} is the use of $\Phi = (\Theta^{\mathrm{node}}, 0)$ for which one recovers the independence model of Section \ref{sec:ord-indep} with tractable normalizing constant. This allows batch sampling of the $p$ independent ordinal variables. Moreover, sampling these independent ordinal variables does not require expensive Markov chain Monte Carlo techniques, which would be the case for annealed importance sampling \citep{neal2001annealed} based Monte Carlo estimates of quantities involving the intractable $Z(\Theta)$. That the variance of the importance sampling estimate is well-behaved for this choice of $\Phi$ was shown in \cite[Proposition 3.2]{chen2024likelihood}. We also remark here that in implementing \eqref{eq:mle_updates}, the choice of $\Phi$ is also updated at each step, namely, we set $\Phi_t = (\Theta_t^{\mathrm{node}}, 0)$.

When $p$ is high, it is natural to assume that $\Theta^{\mathrm{edge}}$ is sparse. In such cases, we add an $\ell_1$-penalty to \eqref{eq:MLE} with penalty parameter $\lambda>0$ and iterate:
\begin{align}\label{eq:pmle_updates}
    \tilde{\Theta}_{t+1} &= \Theta_t + \gamma_t \left[ \sum_{i=1}^n \nabla \log q_{\Theta_t}(y_i) - n T_{N_t}(\Theta_t)\right], \nonumber\\
    \Theta_{t+1}&= \left(\tilde{\Theta}_{t+1}^{\mathrm{node}}, \text{prox}_{\lambda\gamma_t}^{\norm{\cdot}_1}(\tilde{\Theta}^{\mathrm{edge}}_{t+1})\right),
\end{align}
where $\text{prox}_{\lambda \gamma_t}^{\norm{\cdot}_1}$ is the proximal operator of the $\ell_1$-norm which is the soft-thresholding function. Here, we intentionally index $T(\Theta)$ by the number of importance samples used to construct it. Also, note the dependence of $N$ on the iteration number which is important for convergence of the sequence.  The modified updates involve the soft-thresholding operator applied on \eqref{eq:mle_updates}.  We provide the following consistency result, with a proof in Supplementary Section~\ref{sec:supproofs}.
\begin{theorem}\label{thm:mle_convergence}
    Assume there exists a unique maximizer $\Theta^\star$ of the $\ell_1$-penalized log-likelihood $\sum_{i=1}^n \log p_\Theta(y_i) - \lambda\sum_{j<k} |\Theta^{\mathrm{edge}}_{jk}|$, and there exists $c>0$ such that $Z(\Theta)\geq c$ for all $\Theta$. If $\Theta$ lives in a compact set and the step-sizes satisfy $\gamma_t \leq 1/\beta$ for some $\beta>0$ depending on the Lipschitz constant of $\nabla \log Z$, $\sum \gamma_t = \infty$, $\sum{\gamma_t^2} < \infty$, and the Monte Carlo sample size $N_t$ is chosen such that $\sum{\gamma_t}\norm{\text{vec}(b_{N_t}(\Theta_t))} < \infty$, then
    $$\E[\norm{\text{vec}(\Theta_t) - \text{vec}(\Theta^\star)}_2^2] \to 0,$$
    where $\text{vec}(b_{N_t}(\Theta_t))$ is the bias of the ratio-estimator $T_{N_t}(\Theta_t)$ with Monte Carlo sample size $N_t$.
\end{theorem}
With covariates present, the overall procedure remains the same but since $Z(\Theta)$ and $\nabla Z(\Theta)$ are now functions of $X$, we repeat this procedure for every $i = 1, \ldots, n$. The computational cost is $O\!\left(N\left(pM+\tbinom p2\right)\right)$ per gradient step, repeated per subject in the covariate case, but batchable over both $n$ and $N$.

\section{Connections with Pairwise and List-Wise Comparison Models}\label{sec:conn}
In this section we show how preference models such as the BTD model \eqref{eq:BTD_model} and the weak PL model \eqref{eq:DL_model} are connected to the CR--MRF model. In summary, our results show that the BT and the PL models for pairwise and list preferences can be recovered from the CR--MRF model under the assumption that items being rated/compared are \emph{independent}, using an \emph{exponential race} argument. We first focus on the homogeneous case and then discuss the case with covariates as the  arguments carry forward. To begin, consider the homogeneous case with $M = 1$, which is just an Ising model with support $\{0,1\}^p$. We have the following.
\begin{proposition}\label{prop:CR_MRF-BT}
    Consider the homogeneous CR--MRF model \eqref{eq:crmrf-joint} with $M = 1$. Assume $\Theta^{\mathrm{edge}} = 0$, i.e. the variables are independent. Let $\gamma_j = e^{\Theta^{\mathrm{node}}_{j,1}}$ for $j = 1, \ldots, p$.
    Then,  $\P[Y_j < Y_k] \propto \gamma_j, \, \P[Y_j > Y_k] \propto \gamma_k, \, \P[Y_j = Y_k] \propto \nu_{jk}\sqrt{\gamma_j\gamma_k} $ where the normalizing constant is $\gamma_j + \gamma_k + \nu_{jk} \sqrt{\gamma_j \gamma_k}$ and $\nu_{jk} = (1 + \gamma_j \gamma_k)/\sqrt{\gamma_j \gamma_k}\ge 2$.
   
\end{proposition}
The proof is straightforward; but included in the supplement for the sake of completeness. 
In summary, the BTD model with pair-specific tie propensity $\nu_{jk}\ge 2$ can be derived as pairwise comparisons \emph{derived} under an independent CR--MRF model with two levels. The only distinction is that while the BTD model introduces a new global tie parameter, the pairwise comparisons derived here have a tie propensity which is a function of the original parameters, bounded below at 2. 

For the PL model, we begin by the exponential-race representation of the model. Let us consider $p$ players competing in a race such that their arrival times $U_j$ are independent $\text{Exponential}(\gamma_j)$ where $\gamma_j$ are the same strength parameters defined in Section \ref{sec:background}. Then the probability that the $j$-th player wins the race is $\P(U_j < U_k, \, k \neq j)$. It is easily seen that:
\begin{align*}
\P(U_j < U_k, \, k \neq j) &= \int_{0}^\infty \P(U_k> t, \, k \neq j)f_{U_j}(t) dt \\
& = \int_0^\infty \gamma_j e^{-t \sum_{j=1}^p \gamma_j} dt = \dfrac{\gamma_j}{\sum_{j=1}^p \gamma_j}.
\end{align*}
Assuming the first-ranked player arrived at time $t$, the analysis for the next player with the minimum arrival time requires conditioning on the event that $U_k > t$  for all $k \neq j$. But due to the memoryless property of the exponential distribution, one arrives at an expression similar to the above display, only this time the denominator being a sum over the strengths of the players remaining, and the numerator being the strength of the second-ranked player. The exponential race correspondence has been recorded in the literature \citep{diaconis1988group, caron2012bayesian}. However, our key insight lies in relating the exponential race construction to the CR--MRF model, as follows. 

Fix $\Delta>0$ and define $ N_j=\left\lfloor {U_j}/{\Delta}\right\rfloor$ so that $N_j\in\{0,1,2,\ldots\}$ for $j = 1, \ldots, p$. Next define $Y_j = \min\{N_j, M\}$ where $M$ represents the maximum ordinal score of the discrete variable we have considered before. Writing $q_j=e^{-\gamma_j\Delta}$, we get $\P(Y_j=m)=(1-q_j)q_j^m, m=0,\ldots,M-1,$ and $\P(Y_j=M)=q_j^M$. In other words, each $Y_j$ is a truncated geometric random variable. Thus,
\begin{align*}
\log\frac{\P(Y_j=m-1)}{\Pr(Y_j=m)}& =\gamma_j\Delta, \quad m = 1, \ldots, M-1,\\
\log\frac{\P(Y_j=M-1)}{\Pr(Y_j=M)} &=\log(e^{\gamma_j\Delta}-1).
\end{align*}
This means by discretizing the underlying exponential arrival times, we obtain a joint distribution of $(Y_1, \ldots, Y_p)$ which is 
an independent ($\Theta^{\mathrm{edge}} = 0$) CR--MRF model. 
Moreover, the consecutive ratios do not depend on the levels $m = 1, \ldots, M-1$ whereas even in the independent CR--MRF case, a richer class of models is attainable since it allows level-dependent consecutive ratios, i.e., $\Theta^{\mathrm{node}}_{j, m}$ is allowed to vary over $m$. 

In the above, $\Delta$ can be interpreted as the time-grid for discretization. Hence, we fix it to $\Delta = 1$ in the appropriate unit. Additionally, set $Y_j = N_j$, i.e. we remove the right-truncation considered above. The joint distribution of $(Y_1, \ldots, Y_p)$ is simply a distribution of independent Geometric $(q_1, \ldots, q_p)$ random variables. Importantly, all variables $Y_j, \, j = 1, \ldots, p$ satisfy the consecutive ratio and the memoryless property. Moreover, their joint distribution belongs to the exponential family. Next, consider rankings obtained by observed values of $(Y_1, \ldots, Y_p)$. The next result recovers the weak ranking Plackett-Luce model by computing the probability of such events derived from the joint distribution of $(Y_1, \ldots, Y_p)$.
\begin{theorem}\label{thm:CR_MRF-DL}
The variables $(Y_1,\ldots,Y_p)$ induce a weak ranking by grouping items that fall in the same smallest occupied bin. Let $A$ be the set of remaining items and let $\varnothing\neq S\subseteq A$. Define $a_j={(1-q_j)}/{q_j}=e^{\gamma_j}-1$. The probability of any $S \subseteq A$ to be the next tied group is
$$
\P(S\mid A) = \frac{\prod_{j\in S}a_j}
{\sum_{\varnothing\neq B\subseteq A}\prod_{j\in B}a_j}.
$$
After removing $S$, the same rule applies to the remaining items.
\end{theorem}
Thus discretized-race CR--MRF model assigns a tied subset $S$ the stage weight $w_{\mathrm{CR}}(S)=\prod_{j\in S}a_j$, whereas the weak PL model \eqref{eq:DL_model} uses
$ w_{\mathrm{PL}}(S) =\delta_{|S|} (\prod_{j\in S}a_j)^{1/|S|}$. In particular, if all item strengths are equal, $a_j\equiv a$ ($q_j = q$), then
$ w_{\mathrm{CR}}(S)=a^{|S|}, \qquad w_{\mathrm{PL}}(S)=\delta_{|S|}a$.
Hence the two models agree upon taking $\delta_r=a^{r-1}$ and the equal-parameter independent CR--MRF model gives an exact PL weak-ranking model. The original PL model with no ties \eqref{eq:PL_model} is next shown to be a limiting case of the weak-ranking model in Theorem \ref{thm:CR_MRF-DL}.
\begin{corollary}\label{cor:exact_PL}
    Recall $Y_j$ from Theorem \ref{thm:CR_MRF-DL} which was discretized from $U_j \sim \text{Exponential}(\gamma_j)$. Rescale the strength parameters $(\gamma_1, \ldots, \gamma_p)$ to $(e^c \gamma_1, \ldots, e^c \gamma_p)$. Then under the limit $c \to -\infty$, for any permutation $\pi$ of $[p]$,
    $$\P(Y_{\pi_1} < \ldots < Y_{\pi_p}) \rightarrow \prod_{j=1}^{p-1} \dfrac{\gamma_{\pi_j}}{\sum_{k=j}^p \gamma_{\pi_k}}.$$
\end{corollary}
\begin{remark}
    It is important to interpret the setting of Corollary \ref{cor:exact_PL} in terms of the relative strengths of items instead of their absolute strengths. Indeed, as $c \to -\infty$, the absolute strengths approach 0 but the relative strengths stay the same (the common scale $e^c$ cancels out). For PL, only the relative strengths matter not the absolute ones. Further, Corollary~\ref{cor:exact_PL} concerns the uncensored scores $Y_j=N_j$; for a fixed cap $M$, $\P(Y_1=\cdots=Y_p=M)=\exp(-Me^c\sum_j\gamma_j)\to1$, so the limit requires $Me^c\to\infty$.
\end{remark}
Finally, let us consider a covariate-dependent situation. Specifically, let the exponential arrival times depend on a covariate $x$, namely $U_j \sim \text{Exp}(\gamma_j(x))$. Define the discretized arrival times in a similar manner, $Y_j = \lfloor U_j\rfloor$ so that $Y_j \sim \text{Geometric}(q_j(x))$. Thus, the joint distribution of $Y = ( Y_1, \ldots, Y_p)\mid x$ is the product Geometric distribution with parameters $q_j(x), \, j = 1, \ldots, p$. Conditional on $x$, $Y$ satisfies the consecutive ratio assumption and if $\log \gamma_j$ depends linearly on $x$, then we obtain \eqref{eq:cov-coef-node}. Moreover, for each $x$, the derived weak-ranking distribution has the form as in Theorem \ref{thm:CR_MRF-DL}.
\subsection{Information Loss due to Converting Ordinal Data into Ranking}
We formalize the information loss incurred by comparison models due to converting ordinal data to comparison data by coarsening. We use the canonical representation
$p_\eta(y)=\exp\{\eta^\top T(y)-A(\eta)\}$ of Section~\ref{sec:mle}, where
$\eta=\operatorname{vec}(\Theta)$ contains the $pM$ node parameters and
one copy of each of the $\binom p2$ undirected edge parameters.
Let $R=r(Y)$ record all pairwise orderings and ties, but not their rating
levels. Its likelihood under the same CR--MRF is
$p_\eta^R(r)=\sum_{y:r(y)=r}p_\eta(y)$. Then, we have the following result.
\begin{proposition}[Fisher information lost by ranking]
\label{prop:cr-coarsening-information}
Let $p\ge2$, $M\ge1$, and let the node and edge parameters be finite.
Write $\mathcal I_Y(\eta)$ and $\mathcal I_R(\eta)$ for the expected
Fisher information for $\eta$ from one ordinal rating vector and its
complete weak ranking, respectively. Then
\begin{equation}
 \mathcal I_Y(\eta)-\mathcal I_R(\eta)
 =\mathbb E_\eta\!\left[
   \operatorname{Cov}_\eta\{T(Y)\mid R\}\right]\succeq0.
 \label{eq:cr-coarsening-identity}
\end{equation}
If $M\ge2$, the difference is positive definite. Thus, with at least
three rating levels, replacing ratings by their weak ranking loses
Fisher information in every nonzero natural-parameter direction,
including when all ties are retained and $\Theta^{\mathrm{edge}}\ne0$.
\end{proposition}
We note that complete collection of pairwise comparisons contains exactly the same information as $R$; but fitting a product of pairwise likelihoods introduces a separate issue. We provide further details in Supplementary Section \ref{sec:cr-coarsening-race} on explicit calculations of information loss under the exponential race model.
\section{Numerical experiments}\label{sec:simulations}

The objective of our experiments in this section is to validate the \emph{derived} preferences from observed multivariate ordinal data. In particular, we are interested in out-of-sample preference prediction under the covariate-dependent CR--MRF model and the covariate-dependent BT/PL model. By out-of-sample \emph{derived} preference we mean how each model predicts a preference between items $(j,k)$ for a judge with covariate $x$ with ratings by this judge not included in the training data. To this end, the BTD/strict PL model we use introduces dependence of the underlying strength parameters on the covariates. Specifically, the strength of item $j$ compared by a judge with covariate $x$ is modeled as $\log \gamma_j(x) = \alpha_j + \beta_j^{\top} x$. To fit the covariate-dependent comparison models, we extract the pairwise and rank preferences from the observed ordinal scores.
We take $p = 5$ items and $q = 2$ covariates, one categorical encoded by a $\mathrm{Bernoulli}(1/2)$ variable and the other a continuous covariate generated from $\mathrm{N}(0,1)$. Then we generate data from the covariate-dependent CR--MRF \eqref{eq:cov-coef-node} with all $\binom{5}{2}$ edges present at magnitude sampled from $\mathrm{Unif}(1,2)$. These edges are then assigned random signs, e.g. positive has probability $0.2$. We report a five-level scale, $M = 4$, and a binary response, $M = 1$, each replicated ten times with $n = 500$ training and $1{,}000$ test samples. 

Set $\P_{jk}(x) = \P(Y_j > Y_k \mid x) + \tfrac12 \P(Y_j = Y_k \mid x)$. Given a test point $x^\star$ and its corresponding true set of parameters $\Theta_0(x^\star)$, let $\P^{0}_{jk}(x^\star)$ be the true probability of the event. Under an ordinal model which models the scores directly using a covariate-dependent CR--MRF, the estimated probability is then $\hat{\P}_{jk}(x^\star) = \P(Y_j > Y_k \mid \hat{\Theta}(x^\star)) + \tfrac12 \P(Y_j = Y_k \mid \hat{\Theta}(x^\star))$ where $\hat{\Theta}$ is the estimated value of $\Theta$. This can be computed by sampling auxiliary ordinal ratings $Y \sim \text{CR--MRF}(\hat{\Theta}(x^\star))$. In our results, we consider two choices of $\hat{\Theta}$: the MLE and the maximum pseudo-likelihood estimator (MPLE). Another set of estimates of $\P_{jk}(x^\star)$ can be obtained by fitting the comparison models to the derived pairwise and list preferences from the observed scores. Formally, in case of the BTD/strict PL model, we obtain estimates of $\alpha_j, \beta_j$ which can be used to form estimates of $\P_{jk}(x^\star)$ using expressions in Section \ref{sec:background}; for strict PL ties were resolved using stable sort and it estimates $\hat{\P}_{jk}(x)$ as the probability of the event where item $j$ is ranked higher than item $k$. To summarize the results, we report the root mean squared error (RMSE) of $\hat{\P}_{jk}(x_{i})$ from its true value $\P^0_{jk}(x)$, averaged over the $1{,}000$ test judges and the $\binom{5}{2}$ item pairs. We also report the Brier scores of the estimates $\hat{\P}_{jk}(x)$ by comparing it to the actual observed preference outcome. While RMSE measures how well each model recovers the pairwise preference probabilities themselves, the Brier score evaluates how well calibrated these predicted probabilities are in comparison to the actual outcome. Numbers in Table \ref{tab:sim-pref} clearly show the advantage of modeling the ordinal ratings directly as estimated preference probabilities for the MLE provide better approximations and calibration to the actual outcome at the same time.

\begin{table}[!ht]
\scriptsize
\centering
\setlength{\tabcolsep}{1pt}
\caption{ Mean (sd) of RMSE and Brier score of the fitted pairwise preference probabilities $\hat{P}_{jk}(x_{i})$ over $10$ replicates. \emph{Oracle} is based on the true parameters.}
\label{tab:sim-pref}

\scalebox{0.9}{
\begin{tabular}{l cc cc}
\toprule
 & \multicolumn{2}{c}{$M = 4$} & \multicolumn{2}{c}{$M = 1$} \\
\cmidrule(lr){2-3}\cmidrule(lr){4-5}
 & RMSE & Brier & RMSE & Brier \\
\midrule
CR--MRF    & $\mathbf{0.028}$ $(0.002)$& $\mathbf{0.1449}$ $(0.0001)$  & $\mathbf{0.028}$ $(0.002)$ & $\mathbf{0.1215}$ $(0.0001)$ \\
MPLE   & $0.044$ $(0.004)$ & $0.1462$ $(0.0004)$ & $0.034$ $(0.003)$& $0.1219$ $(0.0004)$ \\
PL     & $0.037$ $(0.002)$ & $0.1455$ $(0.0001)$ & $0.039$ $(0.002)$ & $0.1222$ $(0.0001)$ \\
BTD    & $0.031$ $(0.002)$ & $0.1451$ $(0.0001)$ & $0.030$ $(0.001)$ & $0.1216$ $(0.0001)$ \\
\midrule
Oracle & --- & $0.1441$  & --- & $0.1207$ \\
\bottomrule
\end{tabular}
}
\end{table}
\normalsize

\begin{figure*}
    \centering
\includegraphics[width=0.9\textwidth, height = 8.3cm]{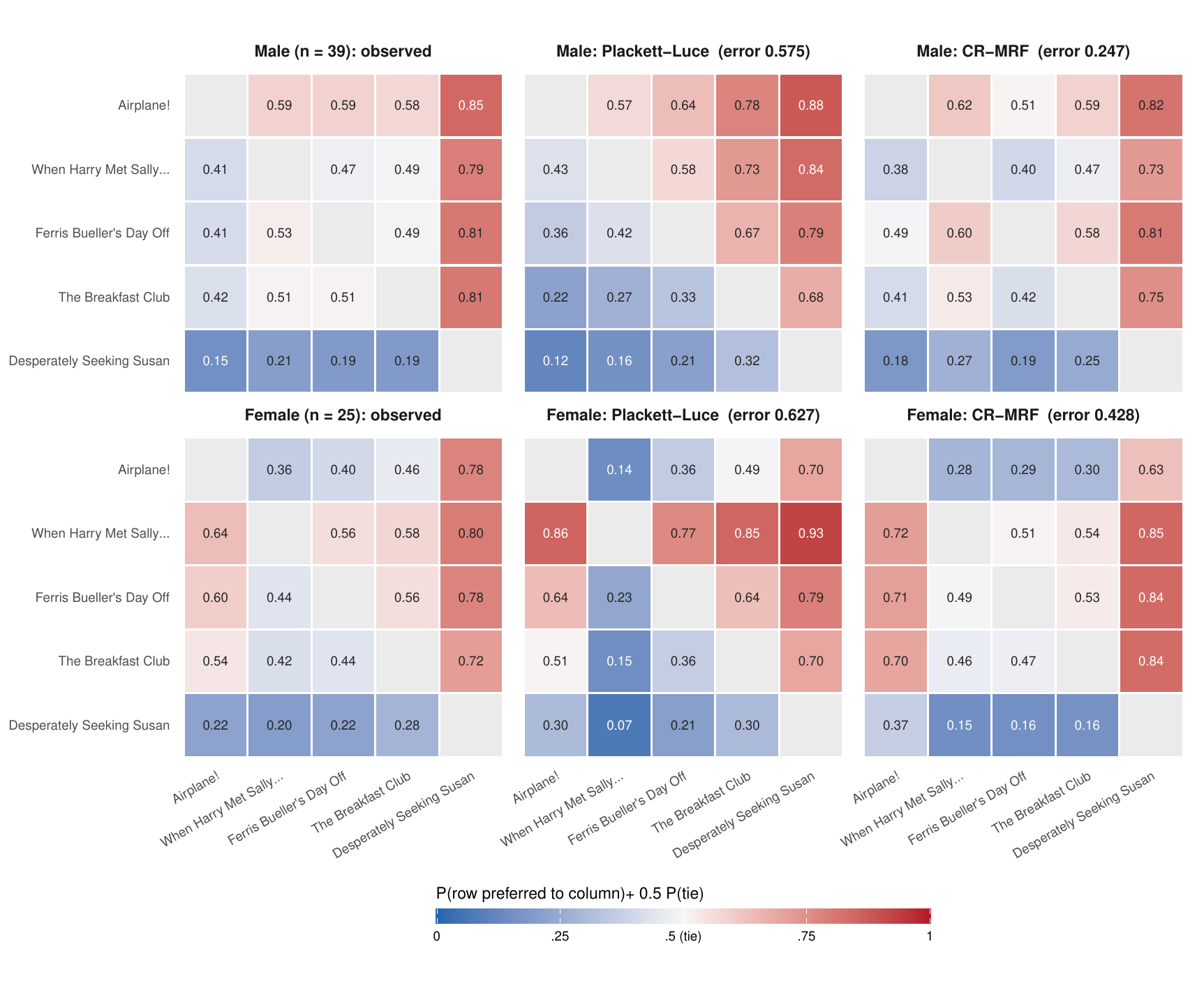}
    \caption{Observed and model-implied pairwise preference probabilities on the 64 held-out MovieLens reviewers, split by reviewer gender, with ordinal ratings ($M = 4$).}
    \label{fig:movielens_summary1}
\end{figure*}
\section{Application to MovieLens}
We use the MovieLens 1M data \citep{harper2015movielens}, which contains $1{,}000{,}209$ ratings on  $\{0,\dots,4\}$ from $6{,}040$ reviewers on $3{,}706$ movies, together with each reviewer's gender and age. Age is recorded in seven groups (under 18, 18--24, \ldots, 56+). We use the group midpoints, standardised, as a continuous covariate. 
We selected $p=5$ widely rated movies. The resulting block contains $n=234$ reviewers who rated all five movies, covering both genders ($76$ female, $158$ male) and all seven age groups.

We shall primarily focus on comparisons with preference models in the main draft. Results from running regression with a covariate dependent CR--MRF model are provided in detail in Section \ref{sec:suppregression} of the Supplement. We split the $n = 234$ data points into a training set of size $n_1 = 170$ training points and $n_2 = 64$ test points. Among the test points there were 39 male and 25 female reviewers. We fit the covariate-dependent CR--MRF model and the covariate-dependent strict PL model with log-strengths depending linearly on the covariates and ties resolved by stable sort. 
Recall $\P_{jk}(x)$ from Section \ref{sec:simulations}. For the 5 movies we obtain $\binom{p}{2}$ pairwise comparisons. Each fitted model provides an estimated $\hat{\P}_{jk}(x^\star)$ where $x^\star$ is a test point which when averaged over all test points provides $\hat{P}_{jk}$. We also compute its observed version $\P_{jk}$ defined as $\P_{jk} = (1/64) \sum_{i =1}^{n_2} \{\mathbb{I}(Y_{ij} > Y_{ik}) + (1/2) \mathbb{I}(Y_{ij} = Y_{ik})\} $. 
These same quantities are also compared by conditioning separately on $x^\star = \{\text{Male}, \text{Female}\}$. The results of $\sum_{j,k}(\P_{jk} - \hat{\P}_{jk})^2$ are summarized in Figure \ref{fig:movielens_summary1}. It shows the observed pairwise comparison probabilities for each gender and the fitted probabilities for all comparisons from the CR--MRF and PL models. Clearly, the comparison probabilities show the proposed model performs better than strict PL. In fact, the results show that PL based estimates of the comparisons are better calibrated for movies with distinct preferences, e.g. `Airplane!' vs `Desperately Seeking Susan' for which the out-of-sample comparison probability is 0.85 for males and 0.78 for females. When the underlying preference is more ambiguous (close to 0.5), then CR--MRF based probabilities perform much better. This is to be expected when underlying ratings information is thrown out and thus ratings for a pair of movies of (1, 4) is treated the same as (3, 4).  Section~\ref{sec:suppcomp} contains additional comparison results with ties broken at random.

\section{\uppercase{Conclusions}}

Pairwise and list-wise comparison problems are classical, but they face a renewed interest due to their importance in the \emph{alignment problem}, where a machine may learn to produce responses that are likely to be rated higher by humans. When full ordinal feedback is available, as in MovieLens or HelpSteer, we establish that the classical models such as BT and PL look at \emph{coarsened} pairwise or list-wise comparisons, which may be wasteful. We then propose an alternative via the CR--MRF model that allows covariate-dependent joint modeling of the full ordinal feedback. Through simulations and data analysis, we establish its benefits in held-out prediction over BT and PL. The intractability of the normalizing constant is a challenge for likelihood-based inference, and we mitigate this building on the approach of \citet{chen2024likelihood}. Future work can consider if the CR--MRF loss can prove to be a reasonable alternative to BT or PL loss in reward models used in the LLM alignment with human feedback.
\section{\uppercase{Supplementary material and Code}}
The supplementary document contains proofs of all theorems and propositions, further details of the CR--MRF model where we specialize the information loss result of Proposition \ref{prop:cr-coarsening-information} for the exponential-race condition and discuss methods to handle missing data, additional numerical experiment and data analysis results. Code and usage examples are available at \href{https://github.com/chenyujie1104/CR-MRF-ordinal-preference-data}{https://github.com/chenyujie1104/CR-MRF-ordinal-preference-data}.

\bibliographystyle{plainnat}
\bibliography{chap3}

\clearpage

\clearpage
\onecolumn
\setcounter{section}{0}
\setcounter{equation}{0}
\setcounter{figure}{0}
\setcounter{table}{0}
\setcounter{page}{1}
\renewcommand{\thepage}{S.\arabic{page}}

\renewcommand{\thesection}{S.\arabic{section}}
\renewcommand{\theequation}{S\arabic{equation}}
\renewcommand{\thefigure}{S\arabic{figure}}
\renewcommand{\thetable}{S\arabic{table}}
\renewcommand{\theHsection}{S.\arabic{section}}
\renewcommand{\theHequation}{S\arabic{equation}}
\renewcommand{\theHfigure}{S\arabic{figure}}
\renewcommand{\theHtable}{S\arabic{table}}

\section*{\uppercase{Supplementary Material}}
\section{Further details of the CR--MRF model}\label{sec:crmrf-extra}
\subsection{Existence of MLE}
For compactness, consider the homogeneous CR--MRF model and let us rewrite the model in the canonical exponential family form as $p_\eta(y) = \exp\{\eta^\T T(y) - A(\eta)\}$ where $\eta = \text{vec}(\Theta)$ and $T(y)$ be the collection of all node and edge-specific sufficient statistics, with $A(\eta) = \log Z(\Theta)$. We assume it is minimal \citep[Chapter 3.2]{wainwright2008graphical}.
Suppose we observe $y_i \overset{iid}{\sim}p_\eta, \, i = 1,\ldots, n$. Then the MLE $\hat{\eta}$ satisfies the score equation $\nabla A(\hat{\eta}) = \overline{T}_n$ where $\overline{T}_n = n^{-1} \sum_{i=1}^n T(y_i)$. But using the properties of exponential family, we know that $\nabla A(\eta) = \E_\eta[T(y)]$. Hence, the MLE $\hat{\eta}$ is the unique point under which the expected sufficient statistic matches the empirical mean of the sufficient statistic. \cite{barndorff2014information} showed that MLE in this case exists if and only if $\overline{T}_n \in \mathrm{ri}(\mathcal{M})$ where $\mathcal{M} = \mathrm{conv}\{T(y): y \in \mathcal{Y}\}$ (the convex hull) and $\mathrm{ri}(A)$ stands for the relative interior for a closed convex set $A$. Conversely, the MLE does not exist if $\bar{T}_n \in \mathcal{M}\,\backslash \,\mathrm{ri}(\mathcal{M})$.
We reformulate this condition for our CR--MRF model in the next result.
\begin{theorem}[Existence of the MLE]\label{thm:mle_existence}
    Suppose we observe $y_i \overset{iid}{\sim}p_\eta, \, i = 1,\ldots, n$. Then the MLE $\hat{\eta} = \argmax_\eta \sum_{i=1}^n \log p_\eta(y_i)$ does not exist if and only if there exists $a \neq 0$ and $b \in \mathbb{R}$ such that $a^\T z \leq b$ for all $z \in \mathcal{M}$ with strict inequality for at least one $z$ and $a^\T T(y_i) = a^\T \overline{T}_n = b$ for all $i = 1, \ldots, n$.
\end{theorem}
A standard situation where this may occur is when all sufficient statistics $T(y_i)$ take the same values for every $y_i$. This condition is a natural extension of the existence of MLE in a Bernoulli model where the MLE (of the natural parameter) does not exist if all observations are either 0 or 1. 
\subsection{Node Conditional Distribution for the Homogeneous CR--MRF}
\label{sec:ord-cond}
 
To derive the conditional distribution of $Y_{j}$ given $Y_{-j} = (Y_{k})_{k \ne j}$, define the neighborhood aggregate
\begin{equation}
\label{eq:r-homog}
  r_{j}(y_{-j})
  \;=\;
  \sum_{k \ne j} \theta_{jk} \, \frac{M - y_{k}}{M}.
\end{equation}
Then,
\begin{equation}
\label{eq:cond-homog}
  p(y_{j} = l \mid y_{-j}, \Theta)
  \;\propto\;
  \exp\!\left\{
    \sum_{m=l+1}^{M} \theta_{j,m}
    \;+\;
    \frac{M - l}{M} \, r_{j}(y_{-j})
  \right\},
  \qquad l = 0, 1, \ldots, M.
\end{equation}
The ratio of two adjacent categories is obtained by subtracting the exponents in \eqref{eq:cond-homog}:
\begin{equation*}
  \log \frac{p(y_{j} = l - 1 \mid y_{-j}, \Theta)}
            {p(y_{j} = l \mid y_{-j}, \Theta)}
  \;=\;
  \left[\, \sum_{m=l}^{M} \theta_{j,m}
        \,-\, \sum_{m=l+1}^{M} \theta_{j,m} \,\right]
  \;+\;
  \left[\, \frac{M - (l - 1)}{M} \,-\, \frac{M - l}{M} \,\right]
    r_{j}(y_{-j}).
\end{equation*}
Substituting the definition of $r_{j}(y_{-j})$ from \eqref{eq:r-homog} yields,
\begin{equation}
\label{eq:adjcat-homog}
  \log \frac{p(y_{j} = l - 1 \mid y_{-j}, \Theta)}
            {p(y_{j} = l \mid y_{-j}, \Theta)}
  \;=\;
  \theta_{j,l}
  \;+\;
  \sum_{k \ne j} \frac{\theta_{jk}}{M^{2}} \, (M - y_{k}),
  \qquad l = 1, \ldots, M.
\end{equation}
This is an adjacent-category logit form. The right-hand side, consisting of a category-specific intercept $\theta_{j,l}$ plus an edge-driven linear predictor in the neighbors of $j$, gives the \emph{consecutive log-odds} ratio between categories $l{-}1$ and $l$. 
The same adjacent-category form is what makes pseudolikelihood inference tractable. Each conditional is an adjacent-category logit regression of $Y_{j}$ on the covariates $x_{k} = M - y_{k}$ for $k \ne j$, with parallel slopes across cutpoints and category-specific intercepts $\theta_{j,1}, \ldots, \theta_{j,M}$. The edge parameter $\theta_{jk}$ is then recovered as the fitted slope on $x_{k}$ multiplied by $M^{2}$. 

\subsection{Node Conditional Distribution for the Covariate-dependent CR--MRF}
\label{sec:ord-cond-cov}
 
Defining the individual-specific neighborhood
aggregate:
\begin{equation}
\label{eq:r-cov}
  r_{j}(y_{i,-j}; x_{i})
  \;=\;
  \sum_{k \ne j} \theta_{jk}(x_{i}) \, \frac{M - y_{i,k}}{M},
\end{equation}
the full conditional of $Y_{i,j}$ given the remaining variables for
individual $i$ is,
\begin{equation}
\label{eq:cond-cov}
  p(y_{i,j} = l \mid y_{i,-j}, x_{i}, \Theta)
  \;\propto\;
  \exp\!\left\{
    \sum_{m=l+1}^{M} \theta_{j,m}(x_{i})
    \;+\;
    \frac{M - l}{M} \, r_{j}(y_{i,-j}; x_{i})
  \right\},
  \qquad l = 0, 1, \ldots, M,
\end{equation}
and the adjacent-category ratio is,
\begin{equation}
\label{eq:adjcat-cov}
  \log \frac{p(y_{i,j} = l - 1 \mid y_{i,-j}, x_{i}, \Theta)}
            {p(y_{i,j} = l \mid y_{i,-j}, x_{i}, \Theta)}
  \;=\;
  \theta_{j,l}(x_{i})
  \;+\;
  \sum_{k \ne j} \frac{\theta_{jk}(x_{i})}{M^{2}} \, (M - y_{i,k}),
  \qquad l = 1, \ldots, M.
\end{equation}

Substituting the linear parameterizations \eqref{eq:cov-coef-node} into \eqref{eq:adjcat-cov} gives the explicit
form of the log-odds ratio,
\begin{align*}
  \log \frac{p(y_{i,j} = l - 1 \mid y_{i,-j}, x_{i}, \Theta)}
            {p(y_{i,j} = l \mid y_{i,-j}, x_{i}, \Theta)}
  &= \theta_{j,l,0}
     \;+\; \sum_{v=1}^{q} \theta_{j,l,v}\, x_{i,v} \\
  &\quad +\; \sum_{k \ne j} \frac{\theta_{jk,0}}{M^{2}}\,(M - y_{i,k}) \\
  &\quad +\; \sum_{k \ne j} \sum_{v=1}^{q}
     \frac{\theta_{jk,v}}{M^{2}}\, x_{i,v}\,(M - y_{i,k}).
\end{align*}
This is an adjacent-category logistic regression of $Y_{i,j}$ on three groups of covariates. The pure covariates $x_{i,v}$ for $v = 1, \ldots, q$ enter with non-parallel slopes $\theta_{j,l,v}$ that
depend on the cutpoint $l$. 
The neighbor values $M - y_{i,k}$ for $k \ne j$ enter with parallel slopes $\theta_{jk,0}/M^{2}$. And the interactions $x_{i,v}(M - y_{i,k})$ enter with parallel slopes $\theta_{jk,v}/M^{2}$, both shared across cutpoints. 
\section{Proofs}\label{sec:supproofs}

\subsection{Proof of Theorem \ref{thm:mle_convergence}}
This Theorem follows from \cite[Theorem 17]{chen2024likelihood} who showed convergence of the iterates for a general class of exponential family graphical models of which CR--MRF is a member. The main difference is that for the CR--MRF case, since the sample space is finite $\nabla \log Z(\Theta)$ has Lipschitz gradients. One still needs the compactness assumption of the parameter space to control the bias of the ratio-estimator of $\nabla \log Z(\Theta)$. Indeed, this implies that there exists importance sample sizes $N_t$ such that over the sequence $\{\Theta_t\}$ obtained by the iterates in \eqref{eq:pmle_updates} such that the corresponding bias $b_{N_t}(\Theta_t)$ in estimating $\nabla \log Z(\Theta_t)$ can be uniformly bounded over $\Theta$.
\subsection{Proof of Proposition \ref{prop:CR_MRF-BT}}
Since $\Theta^{\mathrm{edge}} = 0$, the variables are independent. In fact, with $\gamma_j = e^{\Theta_{j,1}^{\mathrm{node}}}$,  $Y_j \sim \text{Bernoulli}(1/(1+\gamma_j))$.
Hence,
\begin{align*}
    &\P[Y_j < Y_k] = \P[Y_j =0]\P[Y_k = 1] = \dfrac{\gamma_j}{1 + \gamma_j} \dfrac{1}{1+ \gamma_k}, \\
    &\P[Y_j > Y_k] = \P[Y_j = 1]\P[Y_k = 0] = \dfrac{1}{1 + \gamma_j} \dfrac{\gamma_k}{1+ \gamma_k},\\
    &\P[Y_j = Y_k] = \P[Y_j = 1]\P[Y_k = 1] + \P[Y_j = 0]\P[Y_k = 0]  = \dfrac{1 + \gamma_j \gamma_k}{(1 + \gamma_j)(1 + \gamma_k)}. 
\end{align*}
Substituing $\nu_{jk} = (1 + \gamma_j \gamma_k)/\sqrt{\gamma_j \gamma_k}$ yields the expressions in Proposition \ref{prop:CR_MRF-BT}. The AM-GM inequality then gives the lower bound of $\nu_{jk}$.
\subsection{Proof of Theorem \ref{thm:CR_MRF-DL}}
We note that for $S$ to be the next tied group, there must be an $m\ge0$ such that,
$$Y_j=m,\quad j\in S, \qquad Y_k>m,\quad k\in A\setminus S.$$
Hence,
\begin{align*}
\P(S\mid A)
&=\sum_{m=0}^{\infty}
\left\{\prod_{j\in S}(1-q_j)q_j^m\right\}
\left\{\prod_{k\in A\setminus S}q_k^{m+1}\right\}\\
&=\frac{\left(\prod_{j\in A}q_j\right)
\left(\prod_{j\in S}\frac{1-q_j}{q_j}\right)}
{1-\prod_{j\in A}q_j}\\
&=\frac{\prod_{j\in S}a_j}
{\left(\prod_{j\in A}q_j\right)^{-1}-1}.
\end{align*}
The result follows by noting that:
\[
\left(\prod_{j\in A}q_j\right)^{-1}
=e^{\sum_{j\in A}\gamma_j}.
\]
Also,
\[
\sum_{\varnothing\neq B\subseteq A}\prod_{j\in B}a_j
=
\prod_{j\in A}(1+a_j)-1
=
e^{\sum_{j\in A}\gamma_j}-1.
\]
Memoryless property of the Geometric distribution gives the sequential recursion after the selected tied group is removed.

\subsection{Proof of Corollary \ref{cor:exact_PL} }
Set $\epsilon= e^c$ and $a_j(c) = \exp(\epsilon \gamma_j) - 1 =\epsilon \gamma_j+ O(\epsilon^2)$. By Theorem \ref{thm:CR_MRF-DL}, the next nonempty tied group $S \subseteq A$ has probability:
$$\P_ c(S |A) = \dfrac{\prod_{j \in S}  a_j(c)}{\exp(\epsilon \sum_{k \in A} \gamma_k ) - 1}.$$
The denominator is $\epsilon \sum_{k \in A }\gamma_k + O(\epsilon^2)$ . Thus, $ \P_ c(\{j\}|A) = \gamma_j/ \sum_{k \in A} \gamma_k + O(\epsilon)$, whereas $\P_c(S |A) =
O(\epsilon^{|S|  - 1}) $ for $ |S| \geq 2.$  Multiplying the singleton transition probabilities gives the original PL probabilities defined in \eqref{eq:PL_model}. Moreover, there are finitely many remaining sets and at most $p$ stages, so the probability of any tie is $O(\epsilon)$.
\subsection{Proof of Theorem \ref{thm:mle_existence}}
{\bf a) MLE does not exist $\implies a, b$ exists:} From \cite{barndorff2014information}, we know that the MLE does not exist if $\overline{T}_n \in \mathcal{M}\, \backslash \,\mathrm{ri}(\mathcal{M})$. Hence, $\overline{T}_n$ lies on the relative boundary of $\mathcal{M}$. Since $\mathcal{M}$ is closed and convex, by the supporting hyperplane theorem there exists $a$ and $b \in \mathbb{R}$ such that the plane $a^\T z = b$ passes through $\overline{T}_n$, i.e. $a^\T \overline{T}_n = b$. Moreover, this hyperplane satisfies $a^\T z \leq b$ for every $z \in \mathcal{M}$ with strict inequality for at least one $z$. Next, observe that,
\begin{align*}
    b = a^\T \overline{T}_n = \frac{1}{n} \sum_{i=1}^n a^\T T(y_i), \quad a^\T T(y_i) \leq b \,\, \forall i.
\end{align*}
Thus, it must be that $a^\T T(y_i) = b$ for all $i = 1, \ldots, n$.

{\bf b) $a, b$ exists $\implies$ MLE does not exist:} For this direction, we will show if a vector $a$ and a scalar $b$ exists such that $a^\T T(y_i) = a^\T \overline{T}_n  = b$ for all $i = 1, \ldots, n$, and $a^\T z \leq b$ for all $z \in \mathcal{M}$ then the MLE does not exist. Clearly, in this case,
$$\overline{T}_n \in F = \{z \in \mathcal{M}: a^\T z = b\}.$$
Since there is at least one $z \in \mathcal{M}$ such that $a^\T z < b$, $F \subseteq \mathcal{M} \, \backslash \mathrm{ri}(\mathcal{M})$. Hence, the MLE does not exist by \cite{barndorff2014information}.
\subsection{Proof of Proposition~\ref{prop:cr-coarsening-information}}
\label{subsec:cr-coarsening-proof}
Recall the conventions of Section~\ref{sec:crmrf}:
$\mathcal Y=\{0,\ldots,M\}^p$, $\eta=\operatorname{vec}(\Theta)$, and
$A(\eta)=\log Z(\Theta)$. Vectorization includes one coordinate for each
$\theta^{\mathrm{node}}_{j,m}$ and each $\theta^{\mathrm{edge}}_{jk}$,
$j<k$, so that $d:=\dim(\eta)=pM+\binom p2$. The corresponding entries
of $T(y)$ are:
\[
 \phi_{j,m}(y_j)=\mathbf 1\{y_j\le m-1\},\qquad
 \phi_{jk}(y_j,y_k)=\frac{(M-y_j)(M-y_k)}{M^2}.
\]
Finite parameters give positive probability to every $y\in\mathcal Y$.
The weak ranking $R=r(Y)$ records the signs of all $Y_j-Y_k$;
its information is unaffected by reversing which end of the scale is preferred. The proof is then as follows.

\begin{proof}
Let $\mu_\eta=\mathbb E_\eta[T(Y)]=\nabla A(\eta)$, and denote the
likelihood score functions by $\dot\ell_Y$ and $\dot\ell_R$.
For $\mathcal Y_r=\{y:r(y)=r\}$, differentiation of a finite sum gives:
\begin{align}
 \dot\ell_Y(y;\eta)&=T(y)-\mu_\eta,\notag\\
 \dot\ell_R(r;\eta)
 &=\nabla_\eta\log\sum_{y\in\mathcal Y_r}p_\eta(y)\notag\\
 &=\sum_{y\in\mathcal Y_r}
       \frac{p_\eta(y)}{p_\eta^R(r)}\dot\ell_Y(y;\eta)
 =\mathbb E_\eta[\dot\ell_Y(Y;\eta)\mid R=r].
 \label{eq:cr-coarsening-score}
\end{align}
Both score functions have mean zero, so the law of total covariance yields:
\begin{align}
 \mathcal I_Y(\eta)&=\operatorname{Cov}_\eta\{T(Y)\},\notag\\
 \mathcal I_R(\eta)&=\operatorname{Cov}_\eta\{
                         \mathbb E_\eta[T(Y)\mid R]\},\notag\\
 D(\eta):=\mathcal I_Y(\eta)-\mathcal I_R(\eta)
 &=\mathbb E_\eta[\operatorname{Cov}_\eta\{T(Y)\mid R\}].
 \label{eq:cr-coarsening-total-covariance}
\end{align}
Consequently, $a^\top D(\eta)a=0$ if and only if $a^\top T(y)$ is
constant on every ranking class $\mathcal Y_r$. Positivity makes this
an equality at every state in that class.

Suppose $M\ge2$ and this equality holds. Expand each edge product and
absorb its linear terms into the node functions. Subtracting the value
at $0$ permits each node function to be normalized to zero at $0$:
\begin{equation}
 F(y):=a^\top\{T(y)-T(0)\}
 =\sum_{j=1}^p f_j(y_j)+\sum_{j<k}b_{jk}y_jy_k,
 \qquad f_j(0)=0.
 \label{eq:cr-coarsening-additive-quadratic}
\end{equation}
Let $e_j$ be the $j$th coordinate vector and $\mathbf 1$ the all-one vector.
For fixed $j$, the vectors $me_j$, $m=1,\ldots,M$, have the same weak
ranking; therefore $f_j(m)=c_j$ for all $m\ge1$. For $j<k$, the vectors
$e_j+e_k$ and $2e_j+2e_k$ also have the same ranking, including when
$p=2$ and both are all-tied. Their node contributions agree, and their
edge contributions differ by $3b_{jk}$. Hence every $b_{jk}=0$.

The all-tied vectors $0$ and $\mathbf 1$ now give $\sum_j c_j=0$.
Finally, $e_j$ and $\mathbf 1+e_j$ have the same ranking and lie in
$\mathcal Y$ because $M\ge2$. Their values of $F$ give
$c_j=\sum_k c_k=0$ for every $j$. Thus $F$ is identically zero.

The representation is minimal. If a linear combination of $T$ is
constant, evaluation at $M\mathbf 1$ makes the constant zero. Varying
one coordinate with all others at $M$ forces all node coefficients to
vanish; varying two coordinates then forces each edge coefficient to
vanish. It follows that $a=0$, proving $D(\eta)\succ0$.
\end{proof}

\section{Information Loss Calculations under Capped exponential races with unknown common scale}
\label{sec:cr-coarsening-race}
We use the arrival time convention of Section~\ref{sec:conn}: smaller values rank earlier.
Let $U_j\overset{\mathrm{ind}}\sim\operatorname{Exp}(\gamma_j)$ and
$Y_j=\min\{\lfloor U_j/\Delta\rfloor,M\}$, $j=1,2$, with fixed
$\Delta>0$. Writing $t_j=\gamma_j\Delta$ and $q_j=e^{-t_j}$ gives:
\begin{equation}
 \Pr(Y_j=m)=
 \begin{cases}
  (1-q_j)q_j^m,&0\le m<M,\\
  q_j^M,&m=M.
 \end{cases}
 \label{eq:cr-coarsening-capped-geometric}
\end{equation}
This is exactly an independent CR--MRF, with
$\theta^{\mathrm{node}}_{j,m}=t_j$ for $m<M$ and
$\theta^{\mathrm{node}}_{j,M}=\log(e^{t_j}-1)$.
The terminal consecutive-ratio parameter is allowed to differ from the
others. Capping therefore does not violate the CR--MRF specification.

Using the log-worth convention $u_j=\log\gamma_j$ of Section~\ref{sec:background}, write:
\begin{equation}
 u_1=\rho+\delta/2,\qquad u_2=\rho-\delta/2,
 \qquad \delta=\log(\gamma_1/\gamma_2).
 \label{eq:cr-coarsening-race-coordinates}
\end{equation}
Here $\rho$ is an unknown common log-rate. Let $C$ record $Y_1<Y_2$,
$Y_1=Y_2$, or $Y_1>Y_2$. At $\delta=0$, put $t=e^\rho\Delta$,
$q=e^{-t}$, and define:
\begin{align}
 V&=\frac{t^2q(1-q^M)}{(1-q)^2},\qquad
 w=\frac{q(1-q^{2M})}{1+q},\notag\\
 b&=\frac{t}{2(1-q)},\qquad
 h=\frac{\partial\log w}{\partial\rho}
   =t\left\{\frac{2M q^{2M}}{1-q^{2M}}-\frac1{1+q}\right\}.
 \label{eq:cr-coarsening-race-quantities}
\end{align}
The expected information matrices for $(\rho,\delta)$ are
\begin{equation}
 \mathcal I_{Y,(\rho,\delta)}
 =\begin{pmatrix}2V&0\\0&V/2\end{pmatrix},\qquad
 \mathcal I_{C,(\rho,\delta)}
 =\begin{pmatrix}
   2wh^2/(1-2w)&0\\0&2wb^2
  \end{pmatrix}.
 \label{eq:cr-coarsening-two-parameter-information}
\end{equation}

To verify these expressions, the likelihood score for one log-rate is
\[
 \dot\ell_u(y)=
 \begin{cases}
 t\{q/(1-q)-y\},&y<M,\\
 -Mt,&y=M.
 \end{cases}
\]
Its mean is zero and its second moment is
\[
 t^2\left[(1-q)\sum_{m=0}^{M-1}q^m
           \left\{\frac q{1-q}-m\right\}^2+M^2q^M\right]=V.
\]
The pair likelihood scores for $\rho$ and $\delta$ are, respectively,
$\dot\ell_u(Y_1)+\dot\ell_u(Y_2)$ and
$\{\dot\ell_u(Y_1)-\dot\ell_u(Y_2)\}/2$; independence proves the first
matrix in \eqref{eq:cr-coarsening-two-parameter-information}.
For the comparison probabilities, summing a finite geometric series gives
\begin{equation}
 \Pr(Y_1<Y_2)
 =(1-q_1)q_2\frac{1-(q_1q_2)^M}{1-q_1q_2},
 \label{eq:cr-coarsening-win-probability}
\end{equation}
and interchanging indices gives $\Pr(Y_1>Y_2)$. At $\delta=0$ both
equal $w$. The $\delta$-scores of the three outcomes $(<,=,>)$ are
$(b,0,-b)$, and their $\rho$-scores are
$(h,-2wh/(1-2w),h)$. Their weighted outer products give the second
matrix, including its zero off-diagonal entries.

Thus scale and contrast are Fisher-orthogonal at equal rates. Whenever
$h\ne0$, both matrices are nonsingular, and the efficient information for
$\delta$ with $\rho$ unknown equals its diagonal information entry.
Consequently,
\begin{equation}
 \frac{\mathcal I_{C,\delta\mid\rho}}
      {\mathcal I_{Y,\delta\mid\rho}}
 =\frac{1+q^M}{1+q},\qquad \delta=0,\ h\ne0.
 \label{eq:cr-coarsening-exact-ratio}
\end{equation}
The same ratio holds when $\rho$ is known, without the $h\ne0$
qualification. 

This is a local statement at equal
rates, not a uniform efficiency ratio over all parameter values.
Away from equal rates, the off-diagonal information generally does not
vanish, and efficient contrast information must instead be computed as
$\mathcal I_{\delta\delta}-\mathcal I_{\delta\rho}
\mathcal I_{\rho\rho}^{-1}\mathcal I_{\rho\delta}$.

\subsection{Limiting Comparisons}
For $M=1$, the contrast ratio is one at equal rates. For $M\ge2$
and $0<q<1$, it is strictly less than one. At fixed $t>0$, letting
$M\to\infty$ gives $1/(1+q)$. In the joint limit $t\to0$ and
$Mt\to\infty$,
\[
 \mathcal I_{Y,\delta\mid\rho}\longrightarrow\tfrac12,
 \qquad \mathcal I_{C,\delta\mid\rho}\longrightarrow\tfrac14,
 \qquad
 \frac{\mathcal I_{C,\delta\mid\rho}}
      {\mathcal I_{Y,\delta\mid\rho}}\longrightarrow\tfrac12.
\]
The displayed efficient-information limits are along nonsingular points.
Comparison information for the common scale tends to zero as ties vanish.
If $M$ is fixed and $t\to0$, the contrast ratio instead tends to one
while both contrast information quantities tend to zero. This terminal-category
degeneracy is not the uncensored PL limit of Corollary~\ref{cor:exact_PL}.

\section{Handling missing data}\label{sec:missing_data}
The CR--MRF model through its conditional distributions allows for very easy handling of missing data. We provide a brief discussion here. Suppose we observe $y_i^{p \times 1} \mid X_i \sim \text{CR--MRF}(\Theta(X_i))$ independently for $i = 1, \ldots, n$. Assume that for a subset $B \subset \{1, \ldots, n\}$, some components of the $y_{i}, \, i\in B$ are missing, i.e. for $i \in B$, we only observe $y_{D_i}$ where $D_i$ is a proper subset of $\{1, \ldots, p\}$ .  Then for such an $i \in B$, we have the observed likelihood as
\begin{equation}
    p_{obs}(y_{D_i} \mid \Theta(X_i)) = \sum_{y_{D_i^c}} p(y_i \mid \Theta(X_i)), 
\end{equation}
where $p(y_i \mid \Theta(X_i))$ is the full CR--MRF distribution with parameter $\Theta(X_i)$. Suppose the node and edge parameters are parameterized as in \eqref{eq:cov-coef-node}. Stacking all relevant parameters of this model into $\Theta$, the overall log-likelihood and the MLE is then
\begin{align*}
    \ell(\Theta) &= \sum_{i \notin B} \log p(y_i \mid \Theta(X_i)) + \sum_{i \in B } \log p_{obs}(y_{D_i} \mid \Theta(X_i))\\
    \hat{\Theta} & = \argmax \ell(\Theta).
\end{align*}
To obtain the MLE in this case, one can iterate $\Theta_{t+1} = \Theta_{t} + \gamma_t \nabla \ell(\Theta_t)$ where $\nabla \ell(\Theta_t)$ contains two terms - a) gradient over fully-observed data points; this is simple and is what we do in the main draft in Section \ref{sec:mle}, b) gradient of $\log p_{obs}(y_i \mid \Theta(X_i))$ for $i \in B$ which is the main difference in this case. However, it can be obtained using Fisher's identity. For ease of notation, fix $i \in B$ and remove the dependence of $\Theta$ on $X_i$. By Fisher's identity we have
\begin{align*}
\nabla \log p_{obs}(y_D \mid \Theta)& = \E_{p(y_{D^c}\mid y_D, \Theta)}[\nabla \log p(y_D, y_{D^c}\mid \Theta)] \\
& = \E_{p(y_{D^c}\mid y_D, \Theta)} [\nabla \log q_\Theta(y_D, y_{D^c})] - \nabla \log Z(\Theta).
\end{align*}
The second term in the above display can be estimated as detailed in Section \ref{sec:mle}. For the first term, a Monte Carlo estimate can similarly be formed by sampling $y_{D^c}\mid y_D, \Theta$ which requires running a Gibbs sampler over each individual missing coordinate. 

Next, consider a simpler scenario where we have some missing observations in test set but no missing data in training set. This means if $(X_{tr}, Y_{tr})$ is the training set and $(X_{t}, Y_t)$ is the test set, then missingness occurs only in $Y_t$. Then, we can obtain the regression parameters by fitting the covariate-dependent CR--MRF on $(X_{tr}, Y_{tr})$ and on the test data, we can impute the missing values by sampling from  $Y_{D^c, t}\mid Y_{D, t}, \hat{\Theta}(X_{t})$ where $D^c$ again represents the set of missing coordinates in the test data. In Table \ref{tab:missing_data_simulations} we show results from a simulated experiment where $(p, q, M) = (5, 2, 3)$. Data of sizes $n = 550, 1050$ were generated from CR--MRF$(\Theta_0(X))$ and the MLE of model parameters were obtained on a training set of size $(500, 1000)$. In the training set, there were no missing data. However, in the test set (of size $50$ in each case), we considered a missing-at-random scenario where for each $Y$ in the test set a randomly selected subset $D^c$ of $\{1, \ldots, p\}$ ($D$ is the set of observed values) were considered missing. We varied the cardinality of $D^c$ from 1 to 3 and computed the Mean Square Error of imputations after obtaining the MLE of $\Theta_0$ and also conditioning $Y_{D}$. This is computed as
$$\mathrm{MSE} = \frac{1}{50 |D^c|} \sum_{i=1}^{50} \sum_{s \in D_i^c} (Y_{is} - \hat{Y}_{is})^2,$$
where $\hat{Y}_{is}, \, s \in D^c$ is the mean imputed value computed by obtaining 1000 Gibbs samples of $Y_{is}\mid Y_{is'}, \hat{\Theta}(X_i), \, s \in D^c, s' \in D, i = 1, \ldots, 50$. Predictions were also computed using an independent CR--MRF model where because of the independence, predictions in the test set are only conditioned on covariate information. Furthermore, we add out-of-sample pairwise preference comparisons obtained using the covariate dependent BTD model and weak PL model. This is measured in terms of Brier score (lower is better). In all cases, and under both metrics full CR--MRF performs better than other alternatives considered. Additionally, as expected, for larger training size we get smaller MSE/Brier score. MSE/Brier score also increases with the cardinality of $D^c$. 

\begin{table}[]
    \centering
    \scalebox{0.8}{
    \begin{tabular}{ccccccc}
    \toprule
         &  & \multicolumn{2}{c}{MSE} & \multicolumn{3}{c}{Brier score}  \\
         \toprule
         & Size of $D^c$ & Full CR--MRF & Ind. CR--MRF & Full CR--MRF & BTD & Weak PL  \\
         \toprule
        \multirow{3}{*}{$n = 500$} &  $|D^c|$ = 1 & {\bf 0.92} & 1.28 & {\bf 0.57} & 0.61 & 0.59\\
         & $|D^c|$ = 2 & {\bf 0.96} & 1.55 & {\bf 0.62} & 0.66 & 0.64\\
         & $|D^c|$ = 3 &  {\bf 1.04} & 1.92 & {\bf 0.72} & 0.78 & 0.75\\
         \midrule 
         \multirow{3}{*}{$n = 1000$} &  $|D^c|$ = 1 & {\bf 0.75} & 1.13 & {\bf 0.52} & 0.57 & 0.58\\
         & $|D^c|$ = 2 &  {\bf 0.82} & 1.49 & {\bf 0.60} & 0.63 & 0.62\\
         & $|D^c|$ = 3 & {\bf 0.86} & 1.71 & {\bf 0.68} & 0.75 & 0.73\\
         \bottomrule \\
         
    \end{tabular}
    }
    \caption{Mean square error (MSE) in predicting missing values in test data where for each test data the covariate information and a subset of the ordinal variables are observed and the missing values are predicted conditional on these. Brier scores for out-of-sample pairwise preference predictions are also reported. }
    \label{tab:missing_data_simulations}
\end{table}
\section{Additional simulation results}
In this section, we report experimental results comparing the maximum likelihood estimate against the pseudo-likelihood estimate for the CR--MRF model. The pseudo-likelihood estimate proceeds by recognizing that each node-conditional distribution is an adjacent-category logistic regression. The method is implemented using the \texttt{VGAM} \citep{JSSv032i10} package. We consider both the homogeneous ($M = 4$) and covariate-dependent CR--MRF ($M = 2$) models. We choose $p \in \{2, 3, 5\}$, each upper-triangular entry of $\Theta_0^{\mathrm{edge}}, \Theta_0^{\mathrm{node}}$ is independently set to $-2$ with probability $0.5$ and to $0$ otherwise, with the lower triangle filled by symmetry. 

For each $p$, a single realization of $(\Theta^{\mathrm{node}}, \Theta^{\mathrm{edge}})$ is held fixed across replicates, and we simulate $10$ datasets via Gibbs sampling. For the covariate-dependent cases, we retain the above setup and generate the two covariates independently from $\mathrm{N}(0, 1)$, with the nonzero edge slopes fixed at $-1$. For this set of experiments, we did not add the $\ell_1$-penalty.

The  MLE is fit using the construction in Section~\ref{sec:mle}, with $N = 2{,}000$ importance samples per iteration. The results are summarized in Table \ref{tab:ord-lowdim-loss} where we report the node and edge specific Frobenius losses for MLE and pseudo-likelihood. These are defined as $\norm{\hat{\Theta}^{\mathrm{node}} - \Theta^{\mathrm{node}}_0}_F/pM$ and $\norm{\hat{\Theta}^{\mathrm{edge}} - \Theta^{\mathrm{edge}}_0}_F/p^2$.
While in the homogeneous case both the maximum likelihood and the pseudo-likelihood estimates perform equally well, we see clear improvements in the covariate-dependent case for the maximum likelihood estimate; in fact the gap widens as the dimension increases. 

We also benchmarked the MLE against the pseudo-likelihood estimator in high-dimensional setups. Due to unavailability of software implementing a pseudo-likelihood estimate with covariates, here we only consider the homogeneous case. We choose $p \in \{20, 30, 40, 50\}$, each upper-triangular entry of $\Theta_0^{\mathrm{node}}$ and $\Theta_0^{\mathrm{edge}}$ is set to $-2$ with probability $0.05$ and to $0$ otherwise, yielding sparse graphs across a range of sizes. We fixed to $M = 4$.
We report results over 10 independent replicates. The MLE was obtained using the proximal gradient updates \eqref{eq:pmle_updates}, with the regularization parameter $\lambda$ selected by five-fold cross-validation. For these high-dimensional experiments, we used $N = 5{,}000$ importance samples to form the estimate $T(\Theta)$ of $\nabla \log Z(\Theta)$. On the other hand, the pseudo-likelihood estimate was obtained using the \texttt{ordinalNet} package from \cite{JSSv099i06}. Along with the Frobenius losses defined earlier, we also report the model selection results. In particular, we focus on the model selection over $\Theta_0^{\mathrm{edge}}$ and report the precision, recall and Matthews correlation coefficient (MCC). For model selection with the $\ell_1$-penalized MLE we use stability selection \citep{meinshausen2010stability} to combine different estimates from different penalty parameter $\lambda$.

The results of the high-dimensional experiments are summarized in Table \ref{tab:ord-homog-highdim}. Here, we see significant improvements in overall performance for the $\ell_1$-penalized MLE over the penalized pseudo-likelihood estimate. For instance, in terms of Frobenius loss, penalized MLE beats penalized pseudo-likelihood in all dimensions. Moreover, penalized MLE also performs better in terms of model selection as indicated by the higher MCC values. 

\begin{table}[!ht]
\centering\small
\setlength{\tabcolsep}{3.5pt}
\caption{Normalized Frobenius losses over 10 replicates with $n=500$. Left: homogeneous CR--MRF ($M=4$). Right: covariate-dependent CR--MRF ($M=2$, $q=2$). Lower errors are boldened.}
\label{tab:ord-lowdim-loss}
\scalebox{0.9}{
\begin{tabular}{c l ccc ccc}
\hline
 & & \multicolumn{3}{c}{Homogeneous ($M=4$)} & \multicolumn{3}{c}{Covariate ($M=2$)} \\
\cmidrule(lr){3-5}\cmidrule(lr){6-8}
$p$ & Method & Node & Edge & Total & Node & Edge & Total \\
\hline
$2$ & MLE  & 0.09 & 0.03 & 0.04 & {\bf 0.39} & {\bf 0.18} & {\bf 0.28} \\
$2$ & VGAM   & 0.09 & 0.03 & 0.04 & 0.40 & 0.21 & 0.29 \\
$3$ & MLE  & 0.16 & 0.04 & 0.05 & {\bf 0.61} & {\bf 0.81} & {\bf 0.72} \\
$3$ & VGAM  & 0.16 & 0.04 & 0.05 & 0.69 & 1.21 & 0.96 \\
$5$ & MLE  & 0.18 & 0.19 & 0.21 & {\bf 0.73} & {\bf 0.88} & {\bf 0.84} \\
$5$ & VGAM   & 0.18 & 0.19 & 0.21 & 0.93 & 1.35 & 1.21 \\
\hline
\end{tabular}
}
\end{table}

\begin{table}[!ht]
\centering\small
\setlength{\tabcolsep}{3.5pt}
\caption{High-dimensional homogeneous CR--MRF: normalized Frobenius losses and edge-support recovery under stability selection. Means over 10 replicates, $n=500$, $M=4$.}
\label{tab:ord-homog-highdim}
\begin{tabular}{c l ccc ccc}
\hline
 & & \multicolumn{3}{c}{Loss} & \multicolumn{3}{c}{Support recovery} \\
\cmidrule(lr){3-5}\cmidrule(lr){6-8}
$p$ & Method & Node & Edge & Total & Prec. & Rec. & MCC \\
\hline
$20$ & MLE        & {\bf 0.57} & 0.30 & {\bf 0.33} & {\bf 0.76} & 0.77 & {\bf 0.74} \\
$20$ & ordinalNet & 1.90 & {\bf 0.29} & 0.50 & 0.40 & {\bf 0.97} & 0.58 \\
$30$ & MLE        & {\bf 0.71} & 0.39 & {\bf 0.43} & {\bf 0.84} & 0.51 & {\bf 0.63} \\
$30$ & ordinalNet & 2.62 & {\bf 0.38} & 0.61 & 0.54 & {\bf 0.83} & 0.63 \\
$40$ & MLE        & {\bf 0.74} & 0.36 & {\bf 0.40} & {\bf 0.66} & 0.73 & {\bf 0.67} \\
$40$ & ordinalNet & 2.26 & 0.40 & 0.61 & 0.33 & {\bf 0.91} & 0.51 \\
$50$ & MLE        & {\bf 0.80} & 0.40 & {\bf 0.43} & {\bf 0.81} & 0.61 & {\bf 0.69} \\
$50$ & ordinalNet & 2.60 & {\bf 0.37} & 0.54 & 0.44 & {\bf 0.87} & 0.59 \\
\hline
\end{tabular}
\end{table}
\newpage
\section{Additional Data Application Results}\label{sec:movielens_additional_results}
\subsection{CR--MRF regression}\label{sec:suppregression}
Here, we report results for the ordinal regression using the covariate-dependent CR--MRF model on the \texttt{MovieLens} data.
To keep the number of parameters manageable relative to $n$, we impose level-invariant node slopes, $\theta_{j,l}(x_i)=\alpha_{j,l}+\beta_j^\top x_i$, while the edge parameters follow \eqref{eq:cov-coef-node}; the model then has $30$ node and $30$ edge parameters. We fit it on $170$ reviewers and hold out the remaining $64$, using the unpenalised MLE of Section~\ref{sec:mle} with $N=5{,}000$ importance samples per iteration. We also fit the node-conditional maximum pseudo-likelihood estimator (MPLE) with the same parametrization: for each movie, an adjacent-category regression on the covariates, the other ratings $M-y_k$ and their interactions with the covariates, fitted without penalty by \texttt{ordinalNet} \citep{JSSv099i06}. 

For interpretation, recall that
\begin{align*}
\log\frac{\P(Y_j=m-1\mid y_{-j},x)}{\P(Y_j=m\mid y_{-j},x)}
&=\alpha_{j,m}+\beta_j^\top x+\sum_{k\neq j}\theta_{jk}(x)\,\frac{M-y_k}{M^2},
\end{align*}
so a positive $\beta_{j,v}$ multiplies every adjacent-level odds ratio by
$e^{\beta_{j,v}}>1$ and shifts movie $j$ towards lower ratings, and a negative
$\theta_{jk}(x)$ indicates a negative conditional association between movies $j$ and $k$. Table~\ref{tab:ml-node} reports the node slopes. The fitted effects track intuition.  Holding all other $y_k$ fixed at $M$, being female multiplies each adjacent-level odds ratio by $e^{-0.58}\approx0.56$ for \emph{When Harry Met Sally}, a shift towards higher ratings, but by $e^{0.66}\approx1.93$ for \emph{Airplane!}, a shift towards lower ratings. Age separates the titles differently. Older reviewers rate \emph{Desperately Seeking Susan} higher (adjacent-level odds ratio $0.70$ per standard deviation of age) and the two teen comedies, \emph{Ferris Bueller's Day Off} and \emph{The Breakfast Club}, lower (both $1.25$), while age has little effect on \emph{When Harry Met Sally} and \emph{Airplane!}.

\begin{table}[!ht]
\centering\small
\setlength{\tabcolsep}{4pt}
\caption{Estimated node slopes $\beta_j$ for MovieLens 1M under the joint MLE and the node-wise MPLE. Age is standardised.}
\label{tab:ml-node}
\scalebox{0.9}{
\begin{tabular}{l rr rr}
\toprule
 & \multicolumn{2}{c}{Female} & \multicolumn{2}{c}{Age} \\
\cmidrule(lr){2-3}\cmidrule(lr){4-5}
Movie & MLE & MPLE & MLE & MPLE \\
\midrule
When Harry Met Sally\ldots  & $-0.58$ & $-0.61$ & $0.06$  & $-0.36$ \\
Airplane!                   & $0.66$  & $0.69$  & $0.04$  & $-0.03$ \\
Desperately Seeking Susan   & $-0.27$ & $-0.55$ & $-0.36$ & $-0.58$ \\
Ferris Bueller's Day Off    & $-0.03$ & $-1.24$ & $0.22$  & $0.31$  \\
The Breakfast Club          & $-0.29$ & $-0.73$ & $0.22$  & $0.13$  \\
\bottomrule
\end{tabular}
}
\end{table}

The two estimators agree in sign on every gender effect and on the three clearest age effects, and they nearly coincide for the two strongest gender effects, \emph{When Harry Met Sally} and \emph{Airplane!}. Elsewhere the MPLE slopes run larger. The only two sign disagreements, the age effects for \emph{When Harry Met Sally} and \emph{Airplane!}, fall where the MLE slope is near zero. In Figures \ref{fig:edges1}, \ref{fig:edges2} we show the estimated network connections under the two estimators MLE and MPLE conditionally for each covariate subclass (gender by age-group).

\begin{figure}[htbp]
    \centering
\includegraphics[width=0.95\textwidth, height = 18cm]{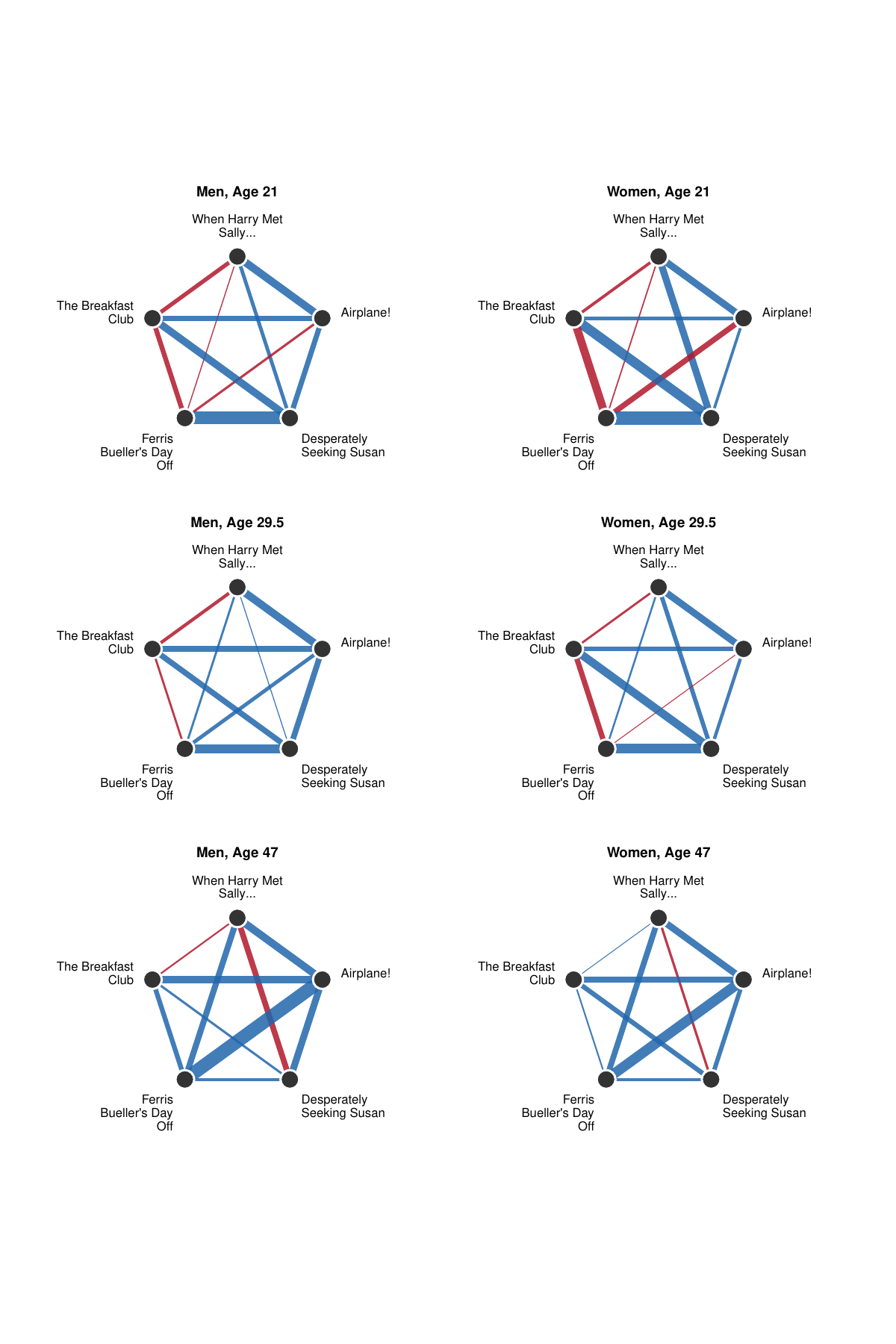}
\vspace{-50pt}
    \caption{Fitted edge network $\theta_{jk}(x)$ under the joint MLE, for men and women at ages 21, 29.5 and 47. Red edges are positive and blue edges negative.
    Line width is proportional to $|\theta_{jk}(x)|$ on a scale shared by all six
    panels.}
    \label{fig:edges1}
\end{figure}

\begin{figure}[htbp]
    \centering
\includegraphics[width=0.95\textwidth, height = 18cm]{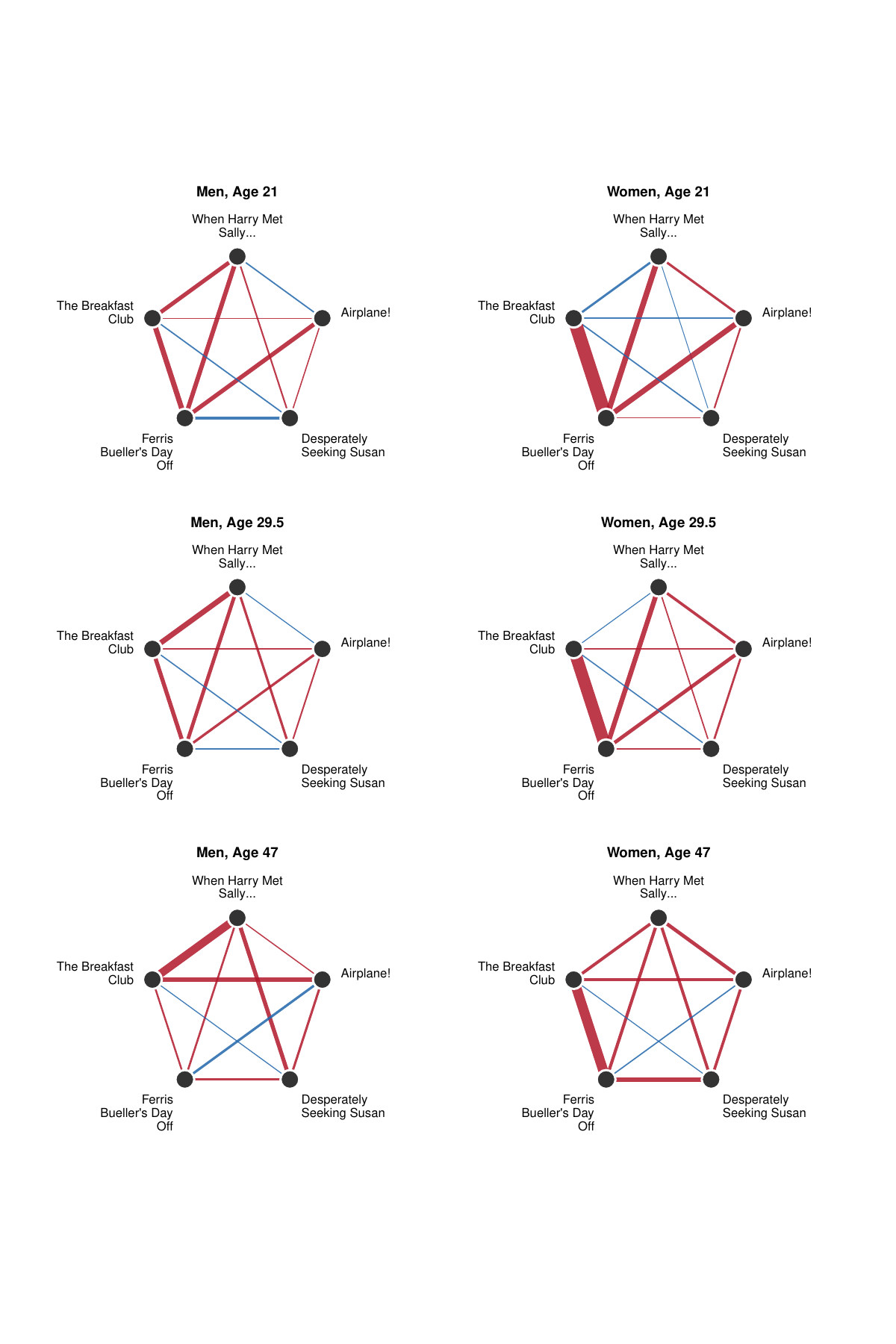}
\vspace{-50pt}
    \caption{Fitted edge network $\theta_{jk}(x)$ under the joint MPLE, for men and women at ages 21, 29.5 and 47. Red edges are positive and blue edges negative.
    Line width is proportional to $|\theta_{jk}(x)|$ on a scale shared by all six
    panels.}
    \label{fig:edges2}
\end{figure}

\subsection{Comparison with preference models}\label{sec:suppcomp}
In the main draft, we showed out-of-sample preference comparison between the covariate-dependent CR--MRF model and a covariate-dependent strict PL model. The ties in that case were resolved by stable sort. To study whether specific tie resolution method has any effect on the result,  here we include results for the strict PL model with ties resolved randomly. For this case, the CR--MRF induced preference probabilities are overall better again (total error for CR--MRF is 0.65 and for strict PL 0.72). Since CR--MRF directly obaserves the ordinal scores, the results for CR--MRF in Figure \ref{fig:movielens_summary1} and \ref{fig:movielens_summary2} should be the same in principal. The difference observed is due to Monte Carlo variation in sampling the out-of-sample $Y$'s given the MLE and the corresponding covariates.

\begin{figure}[htbp]
    \centering
\includegraphics[width=0.95\textwidth, height = 10 cm]{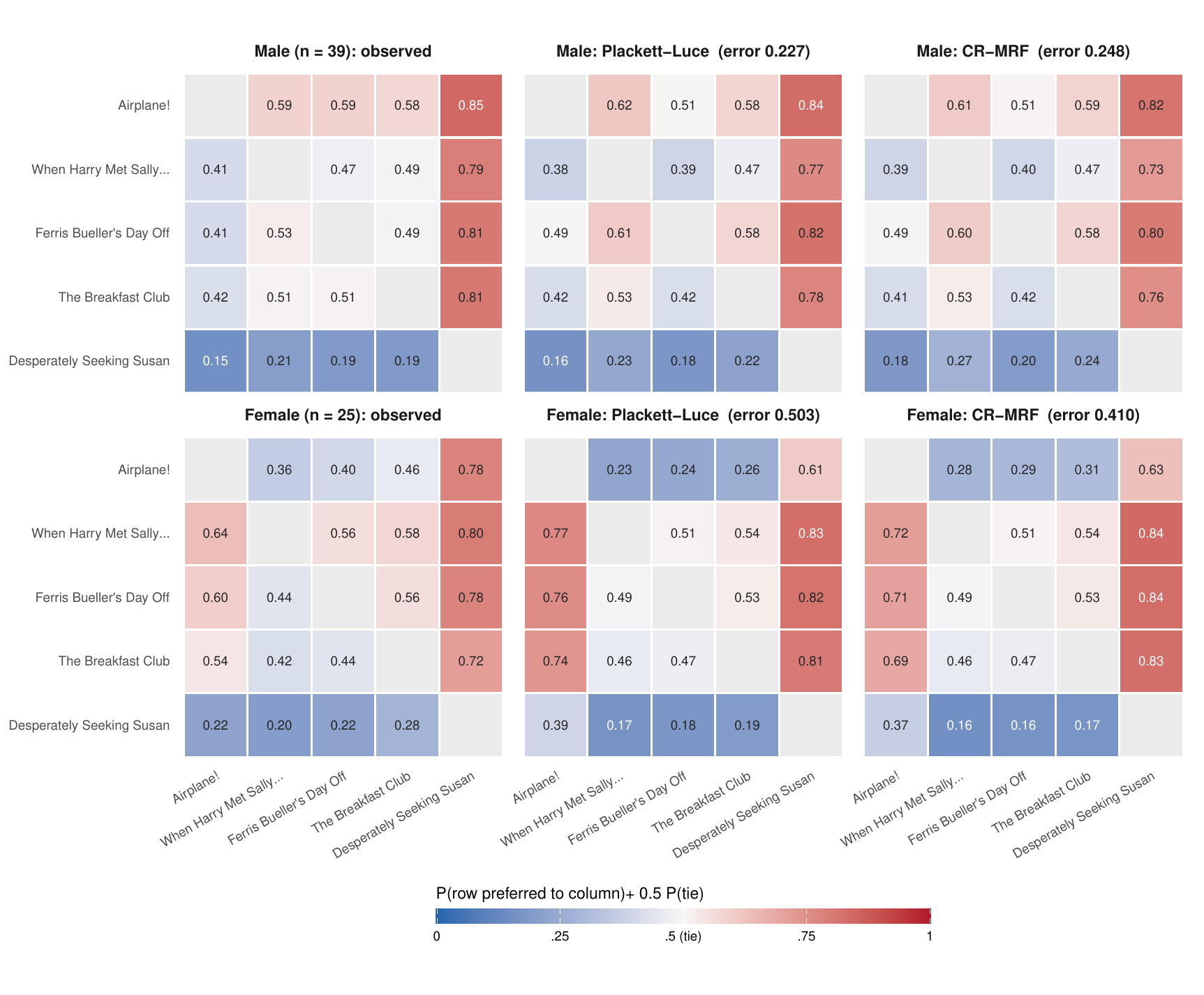}
    \caption{Observed and model-implied pairwise preference probabilities on the 64 held-out MovieLens reviewers, split by reviewer gender, with ordinal ratings ($M=4$) and ties resolved randomly for fitting strict PL.}
    \label{fig:movielens_summary2}
\end{figure}

Finally, we include results where we recode the rankings to a binary scale. In particular, any movie rated higher than 3.5 is coded as 1 and ratings below 3.5 are coded as 0. This allows us to fit a covariate-dependent BTD model to the derived comparisons. Again, we are interested in out-of-sample preference prediction. The results are summarized in Figure \ref{fig:movielens_summary3}. As suggested by Proposition \ref{prop:CR_MRF-BT}, for $M = 1$, the two models behave very similarly which is what we see in Figure \ref{fig:movielens_summary3}.

\begin{figure}[htbp]
    \centering
\includegraphics[width=0.95\textwidth, height = 10cm]{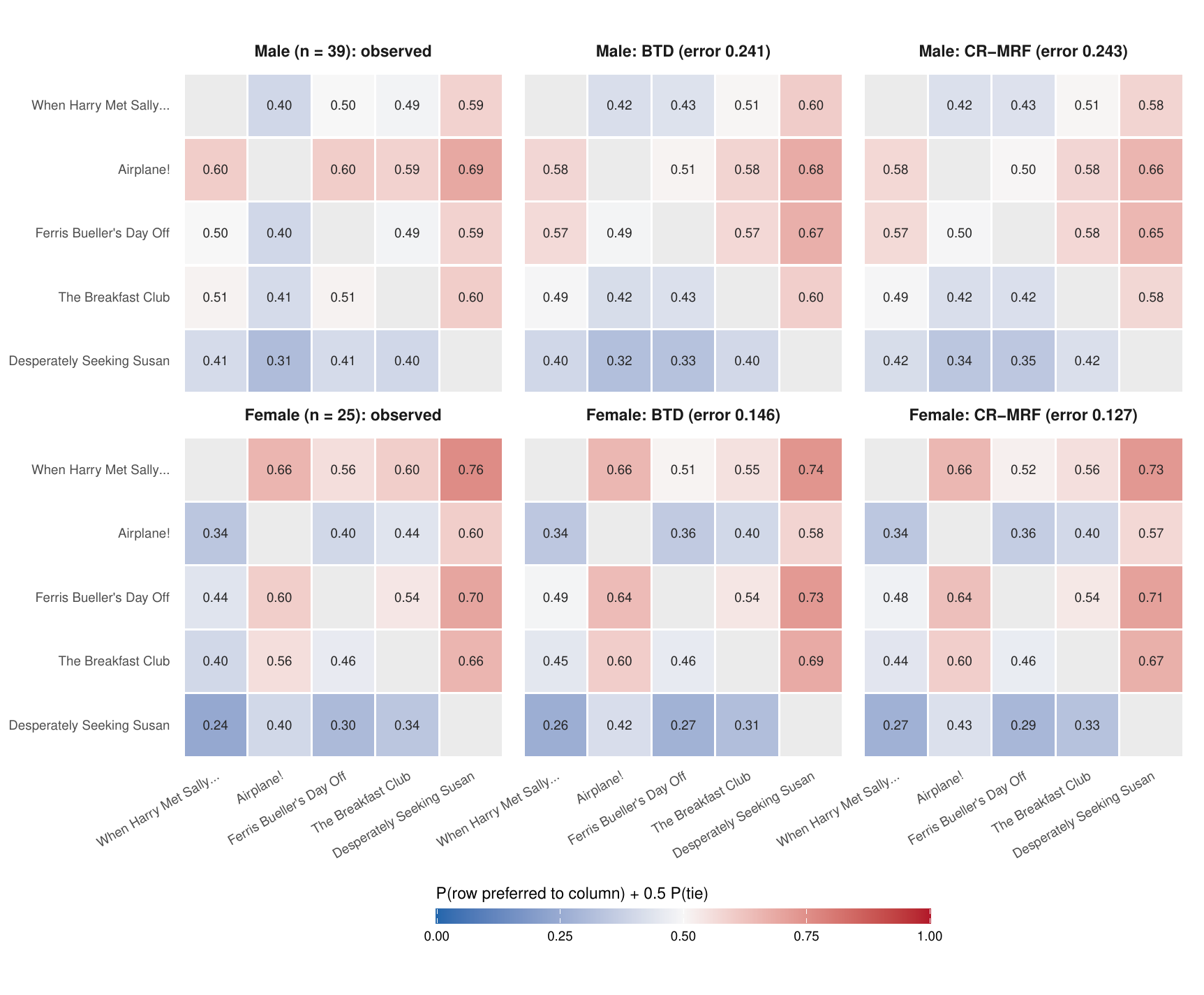}
    \caption{Observed and model-implied pairwise preference probabilities on the 64 held-out MovieLens reviewers, split by reviewer gender, with binary ratings ($M=1$).}
    \label{fig:movielens_summary3}
\end{figure}

\end{document}